\documentclass[lettersize,onecolumn]{IEEEtran}
\usepackage{amsmath,amsfonts}
\usepackage{algorithmic}
\usepackage{algorithm}
\usepackage{array}
\usepackage[caption=false,font=normalsize,labelfont=sf,textfont=sf]{subfig}
\usepackage{textcomp}
\usepackage{stfloats}
\usepackage{url}
\usepackage{verbatim}
\usepackage{graphicx}
\usepackage{cite}
\usepackage{amsmath,amsfonts,amssymb,amsthm,booktabs,color,epsfig,graphicx,hyperref}

\theoremstyle{plain}

\usepackage[T1]{fontenc}
\usepackage{float}
\usepackage{graphicx}
\usepackage{amsmath,amsthm,amssymb,amscd,txfonts}
\usepackage{amsfonts}
\usepackage{amsmath,amssymb,amsthm,enumerate,epsfig,graphicx}
\usepackage{ifthen,latexsym,syntonly}
\usepackage{verbatim}
\usepackage{soul,color}
\usepackage{enumerate}
\usepackage{pstricks}
\usepackage{relsize}

\usepackage{paralist}
\usepackage{multirow}
\usepackage{hhline}
\usepackage{subfig}
\usepackage{epstopdf}

\usepackage{booktabs}
\usepackage{multirow}\usepackage{geometry}
\usepackage{graphicx}\usepackage{threeparttable}

\newtheorem{thm}{Theorem}[section]

\newtheorem{lemma}[thm]{Lemma}

\theoremstyle{definition}

\newtheorem{remark}{Remark}
\newtheorem{assumption}{Assumption}
\newtheorem{theorem}{Theorem}

\DeclareMathOperator{\var}{Var} \DeclareMathOperator{\cov}{Cov}
\DeclareMathOperator{\tr}{tr} 
 \DeclareMathOperator{\diag}{diag}

\renewcommand{\(}{\left(}
\renewcommand{\)}{\right)}

\def\phi{\varphi}

\renewcommand{\refname}{\textbf REFERENCES}

\numberwithin{equation}{section} \numberwithin{thm}{section}

\def\fx {{\mathbf x}}
\def\fZ {{\mathbf Z}}
\def\fz {{\mathbf z}}

\def\fH {{\mathbf H}}

\def\fv {{\mathbf v}}
\def\fu {{\mathbf u}}
\def\fS {{\mathbf S}}

\def\fN {{\mathbf N}}

\def\fC {{\mathbf C}}

\def\fI {{\mathbf I}}

\def\fQ {{\mathbf Q}}

\def\fzero {{\mathbf 0}}

\def\fy {{\mathbf y}}

\newcommand{\bm}[1]{\mbox{\boldmath{$#1$}}}
\newcommand{\bmu}{{\bm \mu}}
\newcommand{\bSig}{{\bm \Sigma}}

\newcommand{\bom}{{\bm \omega}}

\renewcommand{\theenumi}{\arabic{enumi}}
\renewcommand{\theenumii}{\roman{enumii}}

\makeatletter
\renewcommand{\p@enumii}{\theenumi.}

\newcommand{\beq}{\begin{eqnarray}}
\newcommand{\eeq}{\end{eqnarray}}
\newcommand{\beqq}{\begin{eqnarray*}}
	\newcommand{\eeqq}{\end{eqnarray*}}

\begin{document}

\title{Spectrally-Corrected and Regularized QDA Classifier for Spiked Covariance Model}

        
\author{Wenya Luo, Hua Li, Zhidong Bai, Zhijun Liu
       

\thanks{Wenya Luo is with the School of Data Sciences, Zhejiang University of Finance and Economics, 18 Xueyuan Street, Xiasha Higher Education Zone, Hangzhou, Zhejiang, China}
\thanks{Hua Li is with the School of Science, Chang chun University, 6543 Weixing Road, Changchun, Jilin, China}
\thanks{Zhidong Bai is with the KLAS MOE and School of Mathematics and Statistics, Northeast Normal University, 5268 Renmin Street, Changchun, Jilin, China}
\thanks{Zhijun Liu is with     Northeast Normal University, 3-11 Wenhua Road, Heping District, Shenyang, Liaoning, China}}

\markboth{Journal of \LaTeX\ Class Files,~Vol.~1, No.~2, December~2023}%
{Shell \MakeLowercase{\textit{et al.}}: A Sample Article Using IEEEtran.cls for IEEE Journals}

\IEEEpubid{0000--0000~\copyright~2023 IEEE}

\maketitle

\begin{abstract}
Quadratic discriminant analysis (QDA) is a widely used method for classification problems, particularly preferable over Linear Discriminant Analysis (LDA) for heterogeneous data. However, QDA loses its effectiveness in high-dimensional settings, where the data dimension and sample size tend to infinity. To address this issue, we propose a novel QDA method utilizing spectral correction and regularization techniques, termed SR-QDA. The regularization parameters in our method are selected by maximizing the Fisher-discriminant ratio. We compare SR-QDA with QDA, regularized quadratic discriminant analysis (R-QDA), and several other competitors. The results indicate that SR-QDA performs exceptionally well, especially in moderate and high-dimensional situations. Empirical experiments across diverse datasets further support this conclusion.
\end{abstract}


\section{Introduction}\label{sec:1}

\IEEEPARstart{C}{lassification} is a fundamental problem in multivariate statistics and machine learning, with applications ranging from bioinformatics to finance. Quadratic discriminant analysis (QDA) is a popular classification method that allows for distinct covariance matrice between classes, offering greater flexibility than its linear counterpart, Linear discriminant analysis (LDA). Formally, in the ideal setting of two known Gaussian distributions $\fN_p(\bmu_0, \bSig_0)$ (class 0) and $\fN_p(\bmu_1, \bSig_1)$ (class 1), the goal is to classify a new observation $x$, drawn from one of the two distributions with prior probabilities $\pi_0$ and $\pi_1$, into one of the two classes. More details can be found in Anderson \cite{0An}. QDA is powerful in traditional fixed-dimension settings, however, its performance deteriorates significantly in high-dimensional environments where the dimension $p$ is comparable to or exceeds the sample size $n$.


The primary reason of the unsatisfactory performance of QDA for large $p$ is the challenging task of estimating $\bSig_0^{-1}$ and $\bSig_1^{-1}$ in high dimensions, which is a key step in constructing the QDA discriminate function.
To improve the performance of QDA for high dimension, various high-dimensional extensions have been proposed. One strategy often used is to adding some sparsity assumptions. For example, Li and Shao \cite{li2015sparse} assumed that $\bmu_0-\bmu_1$, $\bSig_0$, $\bSig_1$ and $\bSig_0-\bSig_1$ have sparse structures, constructed sparse estimators for these parameters, and plugged them into the QDA discriminate function; Fan et al. \cite{fan2015innovated} proposed a two-stage quadratic discriminant analysis by combining the innovated interaction screening method under assumption that $\bSig_0^{-1}-\bSig_1^{-1}$ is sparse. Similar works include Jiang et al. \cite{jiang2018direct}, Cai and Zhang \cite{cai2021convex} and Gaynanova and Wang \cite{gaynanova2019sparse}. Another common strategy for studying high-dimensional QDA is to apply shrinkage techniques to stabilize estimations of covariance matrices in high dimensions, such as the regularized QDA (R-QDA) proposed by Friedman \cite{friedman1989regularized} and Remey et al. \cite{ramey2016high}. Other related works include Wu et al. \cite{wu2019quadratic}, Wu and Hao \cite{wu2022quadratic} et al. These methods have shown improvements, but they still struggle to some extent with overfitting and the curse of dimensionality when the data does not exhibit obvious sparsity.

Recent advances in large-dimensional random matrix theory (LRMT) provide a promising direction for improving QDA in high dimensions, since it offers various theoretical and methodological support for improving the estimation of population covariance matrices under high dimension situation.

The unsatisfactory performance of QDA for large $p$ is primarily due to the challenging task of estimating $\bSig_0^{-1}$ and $\bSig_1^{-1}$ in high dimensions. Large dimensional random matrix theory (LRMT) provides matured theoretical and methodological support for this task, offering numerous limit theories and statistical methods related to random matrices under the large-dimensional framework, such as \cite{Bai04, bai2012estimation, Baik06, 2011Random}. In recent years, linear and nonlinear classification methods constructed based on LRMT have emerged as a significant research direction. For instance, Houssem et al. \cite{sifaou2020high}, Li et al. \cite{li2022spectrally} and Houssem et al. \cite{Sifaou2020HighDimensionalQD} developed high-dimensional LDA and QDA, respectively, under spiked covariance model, which assumed all eigenvalues of the population covariance matrices are identical except for a finite number. This model, popular in LRMT and described by Johnstone \cite{johnstone2001distribution}, has been applied in various fields, including detection speech recognition \cite{hastie1995penalized, johnstone2001distribution}, mathematical financial \cite{laloux2000random, malevergne210115collective, plerou2002random, kritchman2008determining, passemier2017estimation}, wireless communication \cite{telatar1999capacity}, physics of mixture \cite{sear2003instabilities}, EGG signals\cite{davidson2009functional, fazli2011?1} and statistical learning \cite{hoyle2003limiting}.

In this paper, we assume that the population covariance matrix associated with each class follows a spiked model and propose a spectrally-corrected and regularized QDA (SR-QDA) under the large-dimensional framework. This approach integrates design ideas from the sample spectrally-corrected covariance matrix \cite{li2022spectrally} and the regularized technique \cite{friedman1989regularized}. Specifically, leveraging some theoretical and applied results of the spiked covariance matrix, we use the sample eigenvectors of the extreme sample eigenvalues to provide the directions of the low-rank perturbation, estimate the spiked eigenvalues with the consistent estimator, and regularize them by some common parameters. In this way, compared with Houssem et al. \cite{Sifaou2020HighDimensionalQD}, our method not only retains the spiked structure as much as possible but also reduces the number of parameters to be optimized. We select the optimal regularization parameters by maximizing the asymptotic fisher ratio, which we calculate using tools from LRMT as the sample size and dimension approach infinity at the same rate. Additionally, we present a new estimation for noise variance, provide the asymptotic theory of the estimator in the case of large dimensions, and plug it into our SR-QDA. This further improves the sensitivity of our method to the variance difference between classes, thereby enhancing its classification accuracy. Finally, we compare SR-QDA classifier with other recognized classification techniques, such as QDA, R-QDA, Im-QDA from Houssem et al. \cite{Sifaou2020HighDimensionalQD}, support vector machine (SVM) and k-nearest neighbor (KNN) through some simulation and empirical experiments.

The remainder of this paper is organized as follows: Section~\ref{sec:2} introduces a brief overview of QDA, R-QDA and the basic form of the SR-QDA classifier. Section~\ref{sec:3} present the limit theory of fisher ratio, give a parameter optimization method and an estimation method for variance noise. Section~\ref{sec:4} compares the classification results through simulated data and empirical data respectively. The final section~\ref{sec:5} gives some conclusions.

\section{Methodology}\label{sec:2}

\subsection{QDA and R-QDA}
Quadratic Discriminant Analysis (QDA) is a traditional classification method in multivariate statistics and machine learning. Unlike LDA, which requires all classes share a common covariance matrix, QDA allows each class to have its own covariance matrix, and this flexibility allows it to handle more complex problems. Suppose we have $n=n_0+n_1$ random samples collected from $\pi_0\fN_p(\bmu_0, \bSig_0)+\pi_1\fN_p(\bmu_1, \bSig_1)$, forming two sets $\fC_0=\{\fx_1, \dots, \fx_{n_0}\}$ and $\fC_1=\{\fx_{n_0+1}, \dots, \fx_{n}\}$, satisfying that $\fx_l\in{\fC_i}, i\in{\{0,1\}}$, is drawn from $\fN_p(\bmu_i, \bSig_i)$. Let $\Gamma_i$ denotes the set of indexes of the observations in $\fC_i$. The discriminant function of QDA is given by:
\begin{equation}\label{QDA}
W^{QDA}(\fx)=\eta^{QDA}-\frac{1}{2}(\fx-\bmu_0)^T\bSig_0^{-1}(\fx-\bmu_0)+\frac{1}{2}(\fx-\bmu_1)^T\bSig_1^{-1}(\fx-\bmu_1),
\end{equation}
where $\eta^{QDA}=-\frac{1}{2}\log{\frac{|\bSig_0|}{|\bSig_1|}}-\log{\frac{\pi_1}{\pi_0}}$ and $\pi_i$ is the prior probability for each class. An observation $\fx$ is classified to $\fC_0$ if the discriminant function $W^{QDA}(\fx)$ is positive and to class $\fC_1$ otherwise. In practice, the mean vector $\bmu_i$ and the covariance matrice $\bSig_i$ are unknown and are usually replaced by the following estimators respectively:
$$\bar{\fx}_i=\frac{1}{n_i}\sum\limits_{\ell\in{\Gamma}_i}\fx_{\ell}, \quad \hat{\bSig}_i=\frac{1}{n_i-1}\sum\limits_{\ell\in{\Gamma}_i}(\fx_{\ell}-\bar{\fx}_i)(\fx_{\ell}-\bar{\fx}_i)^T.$$
In the classical asymptotic theory ($p$ fixed, $n\rightarrow \infty$), QDA performs well because the sample covariance matrix is a consistent estimator of the population one. However, its utility diminishes when the dimension approaches the order of magnitude as sample size ($n\rightarrow \infty$, $p\rightarrow \infty$, and $p/n \rightarrow J>0$). This is mainly due to the sample covariance matrix diverging from the population one severely. To overcome this issue, Friedman \cite{friedman1989regularized}, \cite{friedman2001elements} and \cite{zollanvari2015generalized} used ridge esitmators below:
\begin{equation}
\mathbf{H}_i=\left(\mathbf{I}_p+\gamma\hat{\bSig_i}\right)^{-1}, \gamma>0.
\end{equation}
and replaced $\bSig_i^{-1}$ in \eqref{QDA} by $\mathbf{H}_i$ to yield the R-QDA classifier, specifically:
\begin{equation}
\hat{W}^{R-QDA}(\fx)=\eta^{R-QDA}-\frac{1}{2}(\fx-\bar{\fx}_0)^T\mathbf{H}_0(\fx-\bar{\fx}_0)+\frac{1}{2}(\fx-\bar{\fx}_1)^T\mathbf{H}_1(\fx-\bar{\fx}_1),
\end{equation}
where $\eta^{R-QDA}=-\frac{1}{2}\log{\frac{|\mathbf{H}_0|}{|\mathbf{H}_1|}}-\log{\frac{\pi_1}{\pi_0}}$. The classification error of R-QDA corresponding to class $\fC_i$ can be written as:
\begin{equation*}
\epsilon_i^{R-QDA}=\mathbb{P}\left[(-1)^i\hat{W}^{R-QDA}(\fx)<0| \fx\in{\fC_i}\right].
\end{equation*}
And the global classification error is given by:
\begin{equation}
\epsilon^{R-QDA}=\pi_0\epsilon_0^{R-QDA}+\pi_1\epsilon_1^{R-QDA}.
\end{equation}
The optimal parameter of R-QDA classifier $\gamma^*$, that minimizes the global classification error, is generally computed by comparing the performance of a few candidate values using a cross-validation method \cite{friedman1989regularized}.

\subsection{SR-QDA classifier}
In this work, to improve the estimation effect of the sample covariance matrix, we first modify its spectrum, and then use regularization technology to further improve the classification accuracy, so as to put forwad the SR-QDA method. Specifically, we assume the covariance matrix for each class take the following spiked model:
\begin{equation}\label{spikedmodel}
\bSig_i=\sigma_i^{2}\left(\mathbf{I}_{p}+\sum_{j \in \mathbb{I}_i} \lambda_{j,i} \fv_{j,i} \fv_{j,i}^{T}\right), 
\end{equation}
where $\sigma_i^2>0$, $\mathbb{I}_i=\mathbb{I}_{1,i}\cup{\mathbb{I}_{2,i}}$, $\mathbb{I}_{1,i}=\{1,\dots,r_{1,i}\}$, $\mathbb{I}_{2,i}=\{-r_{2,i},\dots,-1\}$, let $r_{i}=r_{1,i}+r_{2,i}$ is a fixed integer denoting the number of spiked eigenvalues of $\bSig_i$, $\lambda_{1,i}\geq{\cdots}\geq{\lambda_{r_{1,i},i}}>0>\lambda_{-r_{2,i},i}\geq{\cdots}\geq{\lambda_{-1,i}}$ ($\lambda_j=\lambda_{p+j+1}$ for $j\in \mathbb{I}_2$), $\mathbf{I}_p$ is a $p\times p$ identity matrix and $\mathbf{v}_{j,i}$ ($j\in \mathbb{I}_i$) is the eigenvector corresponding to the $j$-th largest egenvalue. For simplicity, we assume that $r_{1,i}$ and $r_{2,i}$ are all perfect known. There are several efficient methods for estimating them. Readers can refer to \cite{kritchman2008determining, bai2012estimation, 2020Estimating, ke2021estimation} for more details. For the empirical analysis later in this paper, we use the method given by Ke et al. \cite{ke2021estimation} to estimate $r_{1,i}$ and $r_{2,i}$. And for $\sigma_i$, we give a new estimation method, and related theories are introduce in detail in Section~\ref{sec:3}.
\begin{remark}
	The spiked model shown in \eqref{spikedmodel} is a well-known model in random matrix theory and has been applied to various practical problems.
\end{remark}

\begin{remark}
	Unlike the spiked model studied in \cite{Sifaou2020HighDimensionalQD}, which contains only large spiked eigenvalues, the model here in \eqref{spikedmodel} is more general in that it contains both large spiked eigenvalues and small spiked eigenvalues. In practice, small spiked eigenvalues are also common, and they cannot be ignored in $\bSig_i^{-1}$.
\end{remark}

As we know, the spectrally-corrected method is to correct the spectral elements of the sample covariance matrix to those of $\bSig_i$. So we start from the eigen-decomposition of the sample covariance matrix as $\hat{\bSig}_i=\sum_{j=1}^pl_{j,i}\mathbf{u}_{j,i}\mathbf{u}_{j,i}$, with $l_{j,i}$ being the $j$-th largest eigenvalue of $\hat{\bSig}_i$ and $u_{j,i}$ being its corresponding eigenvector. Then correct $l_{j,i}$ to the corresponding one of $\bSig_i$ as follows:

\begin{equation}\label{SSCM}
\widetilde{\mathbf{S}}_i=\sigma_i^{2}\left(\mathbf{I}_{p}+\sum_{j \in \mathbb{I}_i} \lambda_{j,i} \fu_{j,i} \fu_{j,i}^{T}\right), i\in{\{0,1\}},
\end{equation}
where $\fu_{j,i}=\fu_{p+1+j,i}$ for $j\in{\mathbb{I}_{2,i}}$. Obviously, $\widetilde{\mathbf{S}}_i$ is totally same as  \eqref{spikedmodel} except the spiked eigenvectors. In this way, the original structure of $\bSig_i$ is preserved as much as possible, but the disadvantage is that $\widetilde{\mathbf{S}}_i$ is a biased estimator of $\bSig_i$ because of the bias of sample spiked eigenvectors from that of the population. How to deal with this bias? R-LDA method inspires us to introduce regularization parameters for the sample spiked matrix for QDA. Then we have the spectrally-corrected and regularized QDA (SR-QDA) function, that is 
\begin{equation}\label{SRQDA}
W^{SR-QDA}(\bar{\fx}_0,\bar{\fx}_1,\widetilde{\fH}_0^{-1},\widetilde{\fH}_1^{-1},\fx)=\eta^{SR-QDA}-\frac{1}{2}\left(\fx-\bar{\fx}_0\right)^T\widetilde{\fH}_0^{-1}(\fx-\bar{\fx}_0)+\frac{1}{2}(\fx-\bar{\fx}_1)^T\widetilde{\fH}_1^{-1}(\fx-\bar{\fx}_1),
\end{equation}
where $\eta^{SR-QDA}=-\frac{1}{2}\log{\frac{|\widetilde{\fH}_1^{-1}|}{|\widetilde{\fH}_0^{-1}|}}-\log{\frac{\pi_1}{\pi_0}}$,
\begin{equation}\label{H}
\begin{aligned}
\widetilde{\fH}^{-1}_i&=\sigma_i^{-2}\left(\fI_p+\gamma_{1,i}\sum\limits_{j\in{\mathbb{I}_{1,i}}}\lambda_{j,i}\fu_{j,i}\fu_{j,i}^T+\gamma_{2,i}\sum\limits_{j\in{\mathbb{I}_{2,i}}}\lambda_{j,i}\fu_{j,i}\fu_{j,i}^T\right)^{-1}\\
&=\sigma_i^{-2}\left(\fI_p-\sum\limits_{j\in{\mathbb{I}_{1,i}}}\gamma_{j,i}^{(1)}\fu_{j,i}\fu_{j,i}^T-\sum\limits_{j\in{\mathbb{I}_{2,i}}}\gamma_{j,i}^{(2)}\fu_{j,i}\fu_{j,i}^T\right), 
\end{aligned}
\end{equation}
$\gamma_{j,i}^{(1)}=\frac{\gamma_{1,i}\lambda_{j,i}}{1+\gamma_{1,i}\lambda_{j,i}}, \gamma_{j,i}^{(2)}=\frac{\gamma_{2,i}\lambda_{j,i}}{1+\gamma_{2,i}\lambda_{j,i}}.$
Here $\gamma_{1,i}, \gamma_{2,i}, i=\{0,1\}$ are designed parameters to be optimal. We assume that
\begin{equation}\label{A}
(\gamma_{1,i},\gamma_{2,i})\in{\mathbb{A}}=\left\{\gamma_{1,i}, \gamma_{2,i}: 0<\gamma_{1,i}, \gamma_{2,i}<M, \gamma_{2,i}\notin{\cup_{j\in{\mathbb{I}_{2,i}}}U\left(|\lambda_{j,i}|^{-1}, \delta\right)}\right\},
\end{equation}
for some up bound $M$ and any small $\delta>0$, where $U\left(|\lambda_{j,i}|^{-1}, \delta\right)=\left\{x: \left|x-|\lambda_{j,i}|^{-1}\right|<\delta\right\}$. From a practical point of view, we only need to assume that $\gamma_{2,i}\notin{\cup_{j\in{\mathbb{I}_{2,i}}}U\left(|\lambda_{j,i}|^{-1}, \delta\right)}$ to ensure that the spectral norm of $\widetilde{\fH}_i^{-1}$ is bounded. While, the restriction to range $\mathbb{A}$ is needed later for the proof of the unifrom convergence results. The optimization of $\gamma_{1,i}, \gamma_{2,i}, i=\{0,1\}$ relies on the asymptotic behaior of SR-QDA method under the large-dimensional framework, which is presented in the following part.

\section{Theoretical Results}\label{sec:3}

In this section, we introduce some theretical results, including the asymptotic theory of SR-QDA, regularization parameters optimization and noise variance estimation theory. The asymptotic regime that is considered in our work is described in the following assumptions, for $i\in{\{0,1\}}$:
\begin{assumption}\label{as:1}
	$\fx_{1},\dots, \fx_{n_0}\stackrel{i.i.d.}{\sim}\fN_p(\bmu_0, \bSig_0)$ and $\fx_{n_0+1},\dots, \fx_{n}\stackrel{i.i.d.}{\sim}\fN_p(\bmu_1, \bSig_1)$, where $n=n_0+n_1$.
\end{assumption}

\begin{assumption}\label{as:2}
	$p, n_0, n_1\rightarrow \infty$, $\frac{p}{n_0}\rightarrow J_0>0$, $\frac{p}{n_1} \rightarrow J_1>0$ and $\frac{p}{n}\rightarrow J<1$ .
\end{assumption}
Assumption~\ref{as:2} is the general assumption in the framework of high-dimensional random matrix theory.
\begin{assumption}\label{as:3}
	$r_{1,i}$ and $r_{2,i}$ are fixed and $\lambda_{1,i}>\cdots>\lambda_{r_{1,i},i}>\sqrt{J}>0>-\sqrt{J}
	>\lambda_{-r_{2,i},i}>\cdots>\lambda_{-1,i}>-1$, independently of $p$ and $n$.
\end{assumption}

The assumption~\ref{as:3} is the basis of our analysis and the basic premise of applying high-dimensional random matrix theory, which guarantees a one-to-one mapping between sample eigenvalues and unknown population eigenvalues. In fact, when $\lambda_{j,i}>\sqrt{J}$ (or $-1<\lambda_{j,i}<-\sqrt{J}$), $\lambda_{j,i}$ can be consistently estimated for $j\in \mathbb{I}_i$.

\begin{assumption}\label{as:4}
	$\|\bmu_i\|$ has a bounded Euclidean norm, that is $\|\bmu_i\|=O(1)$, $i=0,1$.
\end{assumption}

\begin{assumption}\label{as:5}
	The spectral norm of $\Sigma$ is bounded, that is $\|\bSig_i\|=O(1)$.
\end{assumption}


According to Assumption \ref{as:1}, the unlabeled observation can be written as $\fx=\bmu_i+\bSig_i^{\frac{1}{2}}\fz, \fz\sim{\fN_p(\fzero,\fI)}$, furthermore for any $\fx\in{\fC_i}$, we obtain
\begin{equation}\label{Yi}
2W^{SR-QDA}(\bar{\fx}_0,\bar{\fx}_1,\widetilde{\fH}_0^{-1}, \widetilde{\fH}_1^{-1}, \fx)=Y_i(\widetilde{\fH}_0^{-1},\widetilde{\fH}_1^{-1})=\fz^TB_i\fz+2\fy_i^T\fz-\xi_i,
\end{equation}
where
$$B_i=\bSig_i^{\frac{1}{2}}\left(\widetilde{\fH}_1^{-1}-\widetilde{\fH}_0^{-1}\right)\bSig_i^{\frac{1}{2}},$$
$$\fy_i=\bSig_i^{\frac{1}{2}}\left[\widetilde{\fH}_1^{-1}(\bmu_i-\bar{\fx}_1)-\widetilde{\fH}_0^{-1}(\bmu_i-\bar{\fx}_0)\right],$$
$$\xi_i=-2\eta^{SR-QDA}+(\bmu_i-\bar{\fx}_0)^T\widetilde{\fH}_0^{-1}(\bmu_i-\bar{\fx}_0)-(\bmu_i-\bar{\fx}_1)^T\widetilde{\fH}_1^{-1}(\bmu_i-\bar{\fx}_1).$$

\begin{theorem}\label{th1}
	Under Assumption \ref{as:1} to \ref{as:5} , we have $$Y_i(\widetilde{\fH}_0^{-1},\widetilde{\fH}_1^{-1})-\widetilde{Y}_i\stackrel{a.s.}{\longrightarrow}0,$$
	where $$\widetilde{Y}_{i}=p \sigma_{i}^{2}\left(\frac{1}{\sigma_{1}^{2}}-\frac{1}{\sigma_{0}^{2}}\right)+\kappa_{i}+2 \fy_{i}^{T} \fz-\xi_{i},$$
	with $$\kappa_{i}= \sum_{j\in{\mathbb{I}_{2,0}}}\frac{\gamma^{(2)}_{j, 0}}{\sigma_{0}^{2}} \left(\fz^{T}\bSig_i^{\frac{1}{2}} \fu_{j, 0}\right)^{2}+\sum_{j\in{\mathbb{I}_{1,0}}}\frac{\gamma^{(1)}_{j,0}}{\sigma_0^2}\left(\fz^T\bSig_i^{\frac{1}{2}}\fu_{j,0}\right)^2-\sum_{j\in{\mathbb{I}_{2,1}}}\frac{\gamma^{(2)}_{j,1}}{\sigma_1^2}\left(\fz^T\bSig_i^{\frac{1}{2}}\fu_{j,1}\right)^2- \sum_{j\in{\mathbb{I}_{1,1}}} \frac{\gamma^{(1)}_{j, 1}}{\sigma_{1}^{2}}\left(\fz^{T}\bSig_i^{\frac{1}{2}} \fu_{j, 1}\right)^{2},$$ where $\gamma^{(1)}_{j,i}=\frac{\gamma_{1,i}\lambda_{j,i}}{1+\gamma_{1,i}\lambda_{j,i}}$, $\gamma_{j,i}^{(2)}=\frac{\gamma_{2,i}\lambda_{j,i}}{1+\gamma_{2,i}\lambda_{j,i}}$, $i=0, 1$.
\end{theorem}

Theorem \ref{th1} indicates that the asymptotic behavior of $Y_i(\widetilde{\fH}_0,\widetilde{\fH}_1)$ corresponds to that of a linear combination of a chi-squared and normal distribution. However, the distribution of $Y_i$ does not have closed form expressions, which makes the analysis of the misclassification rate cumbersome. Therefore, we maximize the fisher ratio metric for parameter optimization. Such a metric quantifies the seperablity between the two classes, by measuring the ratio of the separation between the means to the variance within classes and has been fundamental in the design of the Fisher discriminant analysis (FDA) based classifier. According to \eqref{Yi} and Theorem \ref{th1}, the square root of the fisher ratio associated with the SR-QDA classifer in \eqref{SRQDA} is given by:
$$\rho(\gamma)=\frac{|m_0(\gamma)-m_1(\gamma)|}{\sqrt{\vartheta_0(\gamma)+\vartheta_1(\gamma)}},$$
where $\gamma=(\gamma_{1,1},\gamma_{2,1},\gamma_{1,0},\gamma_{2,0})^T$, $m_i(\gamma)$ and $\vartheta_i(\gamma)$ are respectively the mean and the variance of $\widetilde{Y}_i$ with respect to $\fz$,
given by:
$$m_i(\gamma)=p\sigma_i^2\left(\frac{1}{\sigma_1^2}-\frac{1}{\sigma_0^2}\right)+\mathbb{E}\kappa_i-\xi_i, \quad \vartheta_i(\gamma)=\var(\kappa_i)+4\fy_i^T\fy_i,$$
where we use the fact that $\kappa_i$ and $\fy_i^T\fz$ are uncorrelated. Following the same methodology in the design of FDA, we propose to select $\gamma$ that solves the following optimization problem:
$$\gamma^*==\underset{\gamma}{\operatorname{argmax}} \ \rho(\gamma).$$
The optimization can be performed after these unknown quantities such as $\lambda_{j,i}$ involved in $m_i(\gamma)$ and $\vartheta_i(\gamma)$   being consistently estimated.

For $i=0$ or 1 and $\tilde{i}=1-i$, we define the following quantities:
\begin{eqnarray}\label{preQ}
&\bmu\stackrel{\bigtriangleup}{=}\bmu_1-\bmu_0, \quad \alpha_i=\frac{\|\bmu\|^2}{\sigma_i^2};\\
&a_{j,i}=\frac{\lambda_{j,i}^2-c_i}{\lambda_{j,i}(\lambda_{j,i}+c_i)}, \quad b_{j,i}=\frac{\bmu^T\fv_{j,i}\fv_{j,i}\bmu}{\|\bmu\|^2}, \ j\in{\mathbb{I}_i};\\
&\psi_{\ell,j,\tilde{i},i}=\psi_{j,\ell,i,\tilde{i}}=\fv_{\ell,\tilde{i}}^T\fv_{j,i}, \ \ell\in{\mathbb{I}_{\tilde{i}}}, \ j\in{\mathbb{I}_i};\\
&\phi_{j,i}=1+a_{j,i}\sum\limits_{k\in{\mathbb{I}_{\tilde{i}}}}\lambda_{k,\tilde{i}}\psi_{k,j,\tilde{i},i}^2, \ j \in{\mathbb{I}_i};\\
&\theta_{j,\ell,i}=\sum\limits_{k\in{\mathbb{I}_i}}\lambda_{k,i}\sqrt{a_{j,\tilde{i}}a_{\ell,\tilde{i}}\psi_{j,k,\tilde{i},i}^2\psi_{\ell,k,\tilde{i},i}^2}, \ j,\ell \in{\mathbb{I}_i}.
\end{eqnarray}

Moreover, we shall assume that $\bmu^T\fu_{j,i}>0$and $\bmu^T\fv_{j,i}>0$. This assumption, which is needed to
simplify the presentation of the results, can be made without loss of generality since eigenvectors are defined up to a sign.

\begin{theorem}\label{th2}
	Under the assumption \ref{as:1} to \ref{as:5}, we have 
	\begin{equation}\label{limm}
	m_i(\gamma)-\overline{m}_i(\gamma)\stackrel{a.s.}{\longrightarrow}0,
	\end{equation}
	\begin{equation}\label{limv}
	\vartheta_i(\gamma)-\overline{\vartheta}_i(\gamma)\stackrel{a.s.}{\longrightarrow}0,
	\end{equation}
	where 
	\begin{equation}\label{barm}
	\overline{m}_i(\gamma)=2\eta+c_1-c_0+p\left(\frac{\sigma_i^2}{\sigma_1^2}-\frac{\sigma_i^2}{\sigma_0^2}\right)+(-1)^i\alpha_{\tilde{i}}+g_i^T\tilde{\gamma},
	\end{equation}
	\begin{equation}\label{barV}
	\overline{\vartheta}_i(\gamma)=4\left(\tilde{\gamma}^T\Omega_i \tilde{\gamma}+2e_i^T\tilde{\gamma}+b_i\right),
	\end{equation}
	where $\tilde{i}=1-i$, 
	$\tilde{\gamma}=\left(\gamma_{-r_{2,0},0}^{(2)},\cdots,\gamma_{-1,0}^{(2)},\gamma_{1,0}^{(1)},\cdots,\gamma_{r_{1,0},0}^{(1)},\gamma_{-r_{2,1},1}^{(2)},\cdots,\gamma_{-1,1}^{(2)},\gamma_{1,1}^{(1)},\cdots, \gamma_{r_{1,1},1}^{(1)}\right)^T$,
	
	$$g_i=\left[\left\{\tilde{i}(1+\lambda_{j,0}a_{j,0})+i\left(\frac{\sigma_1^2}{\sigma_0^2}\phi_{j,0}+\alpha_0a_{j,0}b_{j,0}\right)\right\}_{j\in{\mathbb{I}_0}}, \left\{-\tilde{i}\left(\alpha_1a_{j,1}b_{j,1}+\frac{\sigma_0^2}{\sigma_1^2}\phi_{j,1}\right)-i(1+\lambda_{j,1}a_{j,1})\right\}_{j\in{\mathbb{I}_1}}\right]^T$$
	
	$$b_i=c_1\frac{\sigma_i^2}{\sigma_1^2}+c_0\frac{\sigma_i^2}{\sigma_0^2}+\alpha_{\tilde{i}}\frac{\sigma_i^2}{\sigma_{\tilde{i}}^2}\left(1+\sum\limits_{k\in{\mathbb{I}_i}}\lambda_{k,i}b_{k,i}\right),$$
	
	$$e_i=\frac{\alpha_{\tilde{i}}\sigma_i^2}{\sigma_{\tilde{i}}^2}\left[i\left\{-a_{j,0}b_{j,0}-a_{j,0}b_{j,0}^{\frac{1}{2}}\sum\limits_{k\in{\mathbb{I}_1}}\lambda_{k,1}\sqrt{b_{k,1}\psi_{j,k,0,1}^2}\right\}_{j\in{\mathbb{I}_0}}, \tilde{i}\left\{-a_{j,1}b_{j,1}-a_{j,1}b_{j,1}^{\frac{1}{2}}\sum\limits_{k\in{\mathbb{I}_0}}\lambda_{k,0}\sqrt{b_{k,0}\psi_{j,k,1,0}^2}\right\}_{j\in{\mathbb{I}_1}} \right]^T.$$
	
	%

	\begin{equation}
	\Omega_i=\left[\begin{array}{cc}
	\tilde{i}D_0+i\Theta_{0}+iM_0& N_i\\
	N_i^T & iD_1+\tilde{i}\Theta_1+\tilde{i}M_1
	\end{array}\right],
	\end{equation}
	with $D_{i} , \Theta_i,  M_i\in{\mathbb{R}^{{r_i}\times{r_i}}}$ and $N_i\in{\mathbb{R}^{r_0\times{r_1}}}$ defined as, 
	where
	$$D_i=\frac{1}{2}\diag\left\{(1+\lambda_{j,i}a_{j,i})^2\right\}_{j\in{\mathbb{I}_i}},$$
	\begin{equation*}
	\left[\Theta_i\right]_{j,\ell}=\left\{
	\begin{aligned}
	&\frac{\sigma_{\tilde{i}}^4}{2\sigma_{i}^4}\phi_{j,i}^2+\frac{\alpha_i\sigma_{\tilde{i}}^2}{\sigma_i^2}a_{j,i}b_{j,i}, &j=\ell\in{\mathbb{I}_i};\\
	&\frac{\sigma_{\tilde{i}}^4}{2\sigma_{0}^4}\theta_{j,\ell,\tilde{i}}, &j\neq{\ell}\in{\mathbb{I}_i}.
	\end{aligned}
	\right.
	\end{equation*}
	
	$$\left[M_i\right]_{j, \ell}=\alpha_{i} \frac{\sigma_{\tilde{i}}^{2}}{\sigma_i^{2}} a_{j, i} a_{\ell, i} \sqrt{b_{j, i} b_{\ell, i}} \sum\limits_{k\in{\mathbb{I}_{\tilde{i}}}} \lambda_{k, \tilde{i}} \sqrt{\psi_{ j,k, i, \tilde{i} }^2 \psi_{\ell, k,i, \tilde{i} }^2},$$
	$$[N_i]_{j,\ell}=- \frac{\sigma_{\tilde{i}}^{2}}{2\sigma_{i}^{2}}\left(1+i\lambda_{j, 0}+\tilde{i}\lambda_{\ell,1}\right)^{2} a_{j, 0} a_{\ell, 1} \psi_{j, \ell, 0, 1}^{2} , j\in{\mathbb{I}_0}, \ell\in{\mathbb{I}_1}.$$

\end{theorem}
\begin{remark}
	Replacing $\overline{m}_i(\gamma)$ and $\overline{\vartheta}_i(\gamma)$ by their expressions, our optimization problem can be written as:
	\begin{equation}\label{obj:ga}
	\max _{\gamma} \frac{\left|g^{T} \tilde{\gamma}+\beta_{0}+\beta_{1}\right|}{2 \sqrt{\tilde{\gamma}^{T} \Omega \tilde{\gamma}+2 e^{T} \tilde{\gamma}+b}},
	\end{equation}
	where $\beta_i=\alpha_i+p\left(\frac{\sigma_i^2}{\sigma_{\tilde{i}}^2}-1\right),$ $g=g_0-g_1,$ $e=e_0+e_1,$ $\Omega=\Omega_0+\Omega_1$ and $b=b_0+b_1$.
\end{remark}

\begin{theorem}\label{th3}
	Under the setting of Assumptions \ref{as:1} to \ref{as:5},  we have
	\begin{equation}\label{lim_FR}
	\rho(\gamma)-\bar{\rho}(\gamma)\stackrel{a.s.}{\longrightarrow}0, \forall{\gamma}\in{\mathbb{A}},
	\end{equation}
	uniformly, where $\mathbb{A}$ are given in \eqref{A} and  $$\bar{\rho}(\gamma)=\frac{|\bar{m}_0(\gamma)-\bar{m}_1(\gamma)|}{\sqrt{\bar{\vartheta}_0(\gamma)+\bar{\vartheta}_1(\gamma)}}$$.
\end{theorem}

\begin{proof}\label{Proof_th3}
	According to Theorem \ref{th2}, \eqref{lim_FR} holds naturally.
\end{proof}

\begin{theorem}\label{th4}
	Under Assumptions \ref{as:1} to \ref{as:5}, the optimal parameters $\gamma^*=(\gamma_{1,1}^*, \gamma_{2,1}^*, \gamma_{1,0}^*, \gamma_{2,0}^*)^T$ that maximize $\bar{\rho}(\gamma)$ given by
	$$\gamma_{1,1}^*=\frac{\omega_{1,1}^*}{\lambda_{1,1}(1-\omega_{1,1}^*)},\quad \gamma_{2,1}^*=\frac{\omega_{2,1}^*}{\lambda_{-1,1}(1-\omega_{2,1}^*)}, \quad \gamma_{1,0}^*=\frac{\omega_{1,0}^*}{\lambda_{1,0}(1-\omega_{1,0}^*)}, \quad \gamma_{2,0}^*=\frac{\omega_{2,0}^*}{\lambda_{-1,0}(1-\omega_{2,0}^*)},$$
	where $\omega^*=(\omega_{1,1}^*, \omega_{2,1}^*, \omega_{1,0}^*, \omega_{2,0}^*)^T$ is the maximizer of the following function
	$$\widetilde{\rho}(\omega)=\frac{|\widetilde{m}_0(\omega)-\widetilde{m}_1(\omega)|}{\sqrt{\widetilde{\vartheta}_0(\omega)+\widetilde{\vartheta}_1(\omega)}},$$
	and $\omega=(\omega_{1,1}, \omega_{2,1}, \omega_{1,0}, \omega_{2,0})^T\in{\mathbb{B}},$
	\begin{equation*}
	\begin{aligned}
	&\mathbb{B}=\left\{\omega=(\omega_{1,1}, \omega_{2,1}, \omega_{1,0}, \omega_{2,0})^T: 0<\omega_{1,1}, \omega_{2,1}, \omega_{1,0}, \omega_{2,0}<1,\right. \\  &\left. \qquad \quad \omega_{2,1}\neq{(1+\lambda_{j,1}/\lambda_{-1,1})^{-1}}, \omega_{2,0}\neq{(1+\lambda_{k,0}/\lambda_{-1,0})^{-1}},j\in{\mathbb{I}_{2,1}}, k\in{\mathbb{I}_{2,0}}\right\}
	\end{aligned}
	\end{equation*}
	\begin{equation*}
	\begin{aligned}
	&\gamma_{j,0}^{(2)}=\frac{\omega_{2,0}\lambda_{j,0}}{\lambda_{-1,0}(1-\omega_{2,0})-\omega_{2,0}\lambda_{j,0}}, j\in{\mathbb{I}_{2,0}}\quad \gamma_{j,0}^{(1)}=\frac{\omega_{1,0}\lambda_{j,0}}{\lambda_{1,0}(1-\omega_{1,0})+\omega_{1,0}\lambda_{j,0}}, j\in{\mathbb{I}_{1,0}}\\
	&\gamma_{j,1}^{(2)}=\frac{\omega_{2,1}\lambda_{j,1}}{\lambda_{-1,1}(1-\omega_{2,1})-\omega_{2,1}\lambda_{j,1}}, j\in{\mathbb{I}_{2,1}} \quad \gamma_{j,1}^{(1)}=\frac{\omega_{1,1}\lambda_{j,1}}{\lambda_{1,1}(1-\omega_{1,1})+\omega_{1,1}\lambda_{j,1}}, j\in{\mathbb{I}_{1,1}}.
	\end{aligned}
	\end{equation*}
\end{theorem}

Let $ \phi\left(x \right)= x+\frac{J_ix}{x-1}. $ 
From \cite{Baik06}, and under Assumption \ref{as:3}, we have  almost surely, 
$$  
l_{j,i} \rightarrow\sigma_i^{2}\phi\left(\frac{\lambda_{j,i}}{\sigma_i^2}+1\right) = \lambda_{j,i}+\sigma_i^{2}+\sigma_i^{2}J_i\left(1+\frac{\sigma_i^{2}}{\lambda_{j,i}} \right). 
$$
Under the normality assumption on the $  x_{i} $, the maximum likelihood estimator of noise variance $ \sigma_i^{2} $ is $$  \hat{\sigma}_i^{2}=\frac{1}{p-(r_{1,i}+r_{2,i})}\left(\sum_{j=1}^{p}l_{j,i}-\sum_{j=1}^{r_{1,i}}l_{j,i}-\sum_{j=-r_{2,i}}^{-1}l_{j,i} \right).  $$
In the classic setting where the dimension $p$ is relatively small compared with the sample size $n$ (the low dimensional setting), the consistency of $\hat{\sigma}_i^2$ was established in \cite{Anderson1956}. Moreover, it is asymptotically normal with the standard root $n$ convergence rate: as $n \rightarrow \infty$,
$$
\sqrt{ n} \left(\hat{\sigma}_i^2-\sigma_i^2\right) \stackrel{\mathcal{D}}{\rightarrow} \fN\left(0, s^2\right), \quad s^2=\frac{2 \sigma_i^4}{p-(r_{1,i}+r_{2,i})}.
$$
However, when $ p $ is large compared with the sample size $ n $, the above asymptotic result is no longer valid and, indeed, it has been reported in the literature
that $ \hat{\sigma}_i^{2} $ seriously underestimates the true noise variance $ \sigma_i^{2} $; see \cite{kritchman2008determining}. 

Then in the following Theorem \ref{thm:5} and \ref{thm:5.5}, we provide a new estimator of the noise variance  for which rigorous asymptotic theory can be established in the high dimensional setting.

\begin{theorem}\label{thm:5}
	Consider the model with population covariance matrices $ \bSig_i=\sigma_i^{2}(\mathbf{I}_{p}+\sum_{j\in \mathbb{I}_i}\lambda_{j,i} \fv_{j,i}\fv_{j,i}^{T} ) $. 
	Under Assumptions \ref{as:1}--\ref{as:3}, we have
	\begin{align*}
	\frac{p-(r_{1,i}+r_{2,i})}{\sigma_i^{2}\sqrt{2J_{i}}}\left( \hat{\sigma}_i^{2}-\sigma_i^{2}\right) +\sqrt{\frac{J_i}{2}}\left((r_{1,i}+r_{2,i})+\sigma_i^{2}\left(\sum_{j=1}^{r_{1,i}}\frac{1}{\lambda_{j,i}}+\sum_{j=-r_{2,i}}^{-1}\frac{1}{\lambda_{j,i}} \right)  \right) \stackrel{\mathcal{D}}{\rightarrow} \fN\left(0, 1\right).
	\end{align*}
\end{theorem}
\begin{remark}The proof of Theorem \ref{thm:5} is given in the Appendix. From the theorem above, we find there is a bias  $ \sqrt{\frac{c}{2}}((r_{1,i}+r_{2,i})+\sigma_i^{2}(\sum_{j=1}^{r_{1}}\frac{1}{\lambda_{j,i}}+\sum_{j=-r_{2}}^{-1}\frac{1}{\lambda_{j,i}} ) )$ between $ \hat{\sigma}_i^{2} $ and $\sigma_i^{2}$. We can find that the bias also depends on $ \sigma_i^{2} $ we want to estimate. Then it is natural to use a plug-in estimator to correct. The plug-in estimator is  $$ \hat{\sigma}_{*}^{2}=\hat{\sigma}_i^{2}+\frac{\hat{\sigma}_i^{2}J_i}{p-(r_{1,i}+r_{2,i})}\left((r_{1,i}+r_{2,i})+\hat{\sigma}_i^{2}(\sum_{j=1}^{r_{1,i}}\frac{1}{\lambda_{j,i}}+\sum_{j=-r_{2,i}}^{-1}\frac{1}{\lambda_{j,i}} )\right).  $$ 
\end{remark}	
The following CLT is a direct consequence of Theorem \ref{thm:5}.
\begin{theorem}\label{thm:5.5} 
	We assume the same conditions as in Theorem \ref{thm:5}. Then we have
	\begin{align*}	\frac{p-(r_{1,i}+r_{2,i})}{\sigma_i^{2}\sqrt{2J_{i}}}\left( \hat{\sigma}_{*}^{2}-\sigma_i^{2}\right)  \stackrel{\mathcal{D}}{\rightarrow} \mathcal{N}\left(0, 1\right).
	\end{align*}	
\end{theorem}

\begin{theorem}\label{thm:6}
	Under Assumptions \ref{as:1} to \ref{as:4}, we have
	$$\left| \frac{1}{\alpha_i}-\frac{1}{\hat{\alpha}_i}\right|\stackrel{a.s,} {\longrightarrow}{0},\quad \left|\lambda_{j,i}-\hat{\lambda}_{j,i}\right|\stackrel{a.s,} {\longrightarrow}{0},\quad
	\left|b_{j,i}-\hat{b}_{j,i}\right|\stackrel{a.s,} {\longrightarrow}{0}, \quad \left|\psi_{\ell, j, \tilde{i},i}-\hat{\psi}_{\ell, j, \tilde{i},i}\right|\stackrel{a.s,} \longrightarrow 0$$
	where
	\begin{equation} \label{hat_alpha}
	\frac{1}{\hat{\alpha}_i}=\frac{\sigma_i^2}{\left\| \hat{\bmu}\right\|^2-c_1\sigma_1^2-c_0\sigma_0^2};
	\end{equation}
	\beq \label{hat_lammda}
	\hat{\lambda}_{j,i}=-\frac{1}{\sigma_i^2}\left(-\frac{1-c_i}{l_{j,i}}+\frac{1}{n_i}
	\sum_{k\neq{j}}\frac{1}{l_{j,i}-l_{k,i}}\right)^{-1}-1, j\in{\mathbb{I}_i};
	\eeq
	\beq \label{hat_bj}
	\hat{b}_{j,i}=\frac{1+c_i/\hat{\lambda}_{j,i}}{1-c_i/\hat{\lambda}_{j,i}}
	\frac{\hat{\bmu}^T\fu_{j,i}\fu_{j,i}^T\hat{\bmu}}
	{\|\hat{\bmu}\|^2-c_1\sigma_1^2-c_0\sigma_0^2}, j\in{\mathbb{I}_i};
	\eeq
	\beq\label{hat_psi}
	\hat{\psi}_{\ell, j, \tilde{i},i}=\frac{1}{\sqrt{a_{\ell, \tilde{i}} a_{j, i}}} \fu_{\ell, \tilde{i}}^{T} \fu_{j, i}, \ell\in{\mathbb{I}_{\tilde{i}}}, j\in{\mathbb{I}_i};
	\eeq
	with
	$c_0=\frac{p}{n_0}$, $c_1=\frac{p}{n_1}$, $\hat{\bmu}=\overline{\mathbf{x}}_0- \overline{\mathbf{x}}_1$.
\end{theorem}

\begin{proof} [Proof of Theorem \ref{thm:6}] This proof is a direct application of results from ( Theorem 3.1 in Bai and Ding \cite{bai2012estimation}; Theorem 9.1 in Couollet and Debbah \cite{2011Random}; Theorem 9.9 in Baik et al. \cite{2005Phase}), and it's thus omitted here.
\end{proof}


\section{Numerical Results}\label{sec:4}

 In this section, we compare the performance of SR-QDA with QDA, R-QDA, Im-QDA, and other classical classifiers through simulation experiments and empirical analysis.

\subsection{Simulation experiment}
In this part, we first use the following Monte Carlo method to study the different performances of QDA, R-QDA, Im-QDA, and SR-QDA. In this experiment, we maintain the variance heterogeneity of the two classes and focus on the effect of mean differences on these classifiers.

\begin{itemize}
	\item Step 1: Set $p=150$, $\sigma_0^2=1$, $\sigma_1^2=1.5$, $r_{1,0}=r_{1,1}=3$, $r_{2,0}=r_{2,1}=1$, orthogonal directions $\fv_{1,0}$, $\fv_{2,0}$, $\fv_{3,0}$, $\fv_{p,0}$, $\fv_{1,1}$, $\fv_{2,1}$, $\fv_{3,1}$, $\fv_{p,1}$ with the corresponding weights $\lambda_{1,0}=25$, $\lambda_{2,0}=20$, $\lambda_{3,0}=15$, $\lambda_{p,0}=-0.95$, $\lambda_{1,1}=15$, $\lambda_{2,1}=10$, $\lambda_{3,1}=5$, $\lambda_{p,1}=-0.99$. Let $\mu_0=\frac{1}{p}(a,\cdots,a)^{'}$, $\mu_1=\textbf{0}$, and we choose $a=0.5$, $a=0.8$, $a=2$ and $a=2.5$.
	
	\item Step 2: Using the parameters set in Step 1, generate training sample sets from Gaussian distributions with sample sizes $n=100, 200, 300, 400, 500, 600$. Set $\pi_0=0.5$, so that there are $n_0=\pi_0n$ and $n_1=\pi_1n$ training samples for each class.
	
	\item Step 3: Using training sample, estimate the spike eigenvales of the sample covariance and obtain the optimal parameter $\gamma^{*}_{1,0}$, $\gamma^{*}_{2,0}$, $\gamma^{*}_{1,1}$ and $\gamma^{*}_{2,1}$ of the SR-QDA method as discribed in section 3 by using grid search over $\{w_1,w_2\}\in{\{[0,1)\times{[-10,0)}\}}$ to optimize the classification accuracy. Besides, determine the optimal parameter $\gamma^{*}$ of the R-LDA method using grid search over $\gamma^{*}\in{\{10^{i/10},i=-10:1:10\}}$.
	\item Step 4: Calculate the theoretical accuracy rate of the QDA method and estimate the accuracy rates of QDA, R-QDA, Im-QDA and SR-QDA using a set of 2000 test samples.
	
	\item Step 5: Repeat Steps 2-4 500 times and determine the average classification accuracy rate for each classifier.
	
\end{itemize}
Table \ref{tab:simulation_res1} presents the classification accuracy of four methods when applied to simulation data. In this case, as $a$ increases, the difference between $\mu_0$ and $\mu_1$ becomes more pronounced. Regardless of the value of $a$, the QDA method fails completely when $n_0, n_1\leq{p}$ with $p=150$. Although this issue improves with an increase in sample size, the effect remains limited, highlighting the significant impact of high dimensionality on QDA. Similarly, R-QDA's accuracy is also unsatisfactory under various values of $a$ when $n_0, n_1<p$, but its performance is better than that of QDA as the sample size increases. In contrast, Im-QDA and SR-QDA are effective in high-dimensional classification, with SR-QDA showing more pronounced advantages. In addition, the results indicate that SR-QDA performs suboptimally compared to R-QDA when both the difference between $\mu_0$ and $\mu_1$ and the sample size are sufficiently large, such as $a \geq 2$ and $n = 500, 600$. However, our method can outperform others when the sample size is small or moderate, due to the use of variance heterogeneity.

\begin{table}[H]
	\centering
	\caption{Comparison for accuracy rate of QDA, R-QDA, Im-QDA and SR-QDA with different values of $a$.}
	\vskip-0.3cm	
	\smallskip \small
	\label{tab:simulation_res1}
	\begin{threeparttable}
		\resizebox{0.65\textwidth}{!}{%
			\begin{tabular}{@{}llllllll@{}}
				\toprule
				&$n$                 & $100$           & $200$           & $300$           & $400$           & $500$           & $600$           \\ \midrule
				\multirow{4}{*}{$a=0.5$} & QDA             & 0.5000          & 0.5000          & 0.4995          & 0.5325          & 0.5525          & 0.5739          \\
				& R-QDA           & 0.5609          & 0.5937          & 0.6138          & 0.6323          & 0.6519          & 0.6702          \\
				& Im-QDA          & 0.6978          & 0.6936          & 0.6950          & 0.6966          & 0.6975          & 0.6985          \\
				& \textbf{SR-QDA} & \textbf{0.7428} & \textbf{0.7539} & \textbf{0.7491} & \textbf{0.7554} & \textbf{0.7381} & \textbf{0.7551} \\ \midrule
				\multirow{4}{*}{$a=0.8$} & QDA             & 0.5000          & 0.5000          & 0.5004          & 0.535           & 0.5567          & 0.5797          \\
				& R-QDA           & 0.5743          & 0.6091          & 0.6324          & 0.6524          & 0.6718          & 0.6866          \\
				& Im-QDA          & 0.7050          & 0.7023          & 0.7034          & 0.7041          & 0.7046          & 0.7055          \\
				& \textbf{SR-QDA} & \textbf{0.7306} & \textbf{0.7343} & \textbf{0.7468} & \textbf{0.7492} & \textbf{0.7396} & \textbf{0.7449} \\ \midrule
				\multirow{4}{*}{$a=2$}   & QDA             & 0.5000          & 0.5000          & 0.5001          & 0.5582          & 0.5883          & 0.6175          \\
				& R-QDA           & 0.6628          & 0.7170          & 0.7470          & 0.7670          & 0.7822          & 0.7919          \\
				& Im-QDA          & 0.7404          & 0.7407          & 0.7407          & 0.7420          & 0.7429          & 0.7432          \\
				& \textbf{SR-QDA} & \textbf{0.7781} & \textbf{0.7711} & \textbf{0.7960} & \textbf{0.7872} & \textbf{0.7690} & \textbf{0.7739} \\ \midrule
				\multirow{4}{*}{$a=2.5$} & QDA             & 0.5000          & 0.5000          & 0.5017          & 0.5758          & 0.6078          & 0.6424          \\
				& R-QDA           & 0.7103          & 0.7668          & 0.7974          & 0.8166          & 0.8299          & 0.8391          \\
				& Im-QDA          & 0.7601          & 0.7606          & 0.7622          & 0.7632          & 0.7633          & 0.7636          \\
				& \textbf{SR-QDA} & \textbf{0.8016} & \textbf{0.7792} & \textbf{0.7847} & \textbf{0.7990} & \textbf{0.7855} & \textbf{0.7745} \\ \bottomrule
			\end{tabular}%
		}
		\begin{tablenotes}
			\footnotesize
			\item[-] $p=150$, $\sigma_0^2=1$, $\sigma_1^2=1.5$. 
			\item[-] $n$ is the total training sample size of the two classes, and the proportion of $\fC_0$ is $\pi_0=0.5$.
		\end{tablenotes}
	\end{threeparttable}
\end{table}

As a second investigation, we study the impact of the difference between $\sigma_0^2$ and $\sigma_1^2$ using the same Monte Carlo method as described above, except for setting $a=0.5$ and $\sigma_1^2=1.2, 1.5, 2, 4$. Table \ref{tab:simulation_res2} reports the accuracy rates of QDA, R-QDA, Im-QDA and SR-QDA for fixed $\sigma_0^2$ and varying values of $\sigma_1^2$. It can be seen that SR-QDA outperforms the other methods in all cases. This result is expected since as $|\sigma_0^2-\sigma_1^2|$ increases, the two classes become more distinguishable, leading to better performances. Other methods do not make use of this difference well and often exhibit higher estimation errors for the covariance matrices, which directly affects their performance. 

\begin{table}[H]
	\centering
	\caption{Comparison for accuracy rate of QDA, R-QDA, Im-QDA and SR-QDA with different values of $\sigma_1^2$.}
	\vskip-0.3cm	
	\smallskip \small
	\label{tab:simulation_res2}
	\begin{threeparttable}
		\resizebox{0.65\textwidth}{!}{%
			\begin{tabular}{@{}clcccccc@{}}
				\toprule 
				& $n$             & $100$           & $200$           & $300$           & $400$           & $500$           & $600$           \\ \midrule
				\multirow{4}{*}{$\sigma_1^2=1.2$} & QDA             & 0.5000          & 0.5000          & 0.4992          & 0.5256          & 0.5405          & 0.5524          \\
				& R-QDA           & 0.5496          & 0.5630          & 0.5719          & 0.5789          & 0.5835          & 0.5855          \\
				& Im-QDA          & 0.5293          & 0.5512          & 0.5703          & 0.5728          & 0.5725          & 0.5715          \\
				& \textbf{SR-QDA} & \textbf{0.5732} & \textbf{0.6101} & \textbf{0.5956} & \textbf{0.6159} & \textbf{0.6220} & \textbf{0.6219} \\ \midrule
				\multirow{4}{*}{$\sigma_1^2=1.5$} & QDA             & 0.5000          & 0.5000          & 0.4995          & 0.5325          & 0.5525          & 0.5739          \\
				& R-QDA           & 0.5609          & 0.5937          & 0.6138          & 0.6323          & 0.6519          & 0.6702          \\
				& Im-QDA          & 0.6978          & 0.6936          & 0.6950          & 0.6966          & 0.6975          & 0.6985          \\
				& \textbf{SR-QDA} & \textbf{0.7428} & \textbf{0.7539} & \textbf{0.7491} & \textbf{0.7554} & \textbf{0.7381} & \textbf{0.7551} \\ \midrule
				\multirow{4}{*}{$\sigma_1^2=2$}   & QDA             & 0.5000          & 0.5000          & 0.5001          & 0.5269          & 0.5445          & 0.5756          \\
				& R-QDA           & 0.6223          & 0.6903          & 0.7374          & 0.7754          & 0.8148          & 0.8454          \\
				& Im-QDA          & 0.8397          & 0.8421          & 0.8438          & 0.8455          & 0.8465          & 0.8468          \\
				& \textbf{SR-QDA} & \textbf{0.8585} & \textbf{0.8555} & \textbf{0.8521} & \textbf{0.8628} & \textbf{0.8750} & \textbf{0.8814} \\ \midrule
				\multirow{4}{*}{$\sigma_1^2=4$}   & QDA             & 0.5000          & 0.5000          & 0.4999          & 0.5066          & 0.5542          & 0.6817          \\
				& R-QDA           & 0.8480          & 0.9006          & 0.9166          & 0.9134          & 0.9333          & 0.9386          \\
				& Im-QDA          & 0.8911          & 0.8954          & 0.9116          & 0.9219          & 0.9320          & 0.9325          \\
				& \textbf{SR-QDA} & \textbf{0.9518} & \textbf{0.9636} & \textbf{0.9634} & \textbf{0.9640} & \textbf{0.9613} & \textbf{0.9619} \\ \bottomrule
			\end{tabular}%
		}
		\begin{tablenotes}
			\footnotesize
			\item[-] $p=150$, $a=0.5$, $\sigma_0^2=1$.
			\item[-] $n$ is the total training sample size of the two classes, and the proportion of $\fC_0$ is $\pi_0=0.5$.
		\end{tablenotes}
		
	\end{threeparttable}
\end{table}

\begin{figure}[ht]
	\centering
		\subfloat[\label{fig:a}$\sigma_1^2=1.2$]{\includegraphics[width=0.48\textwidth]{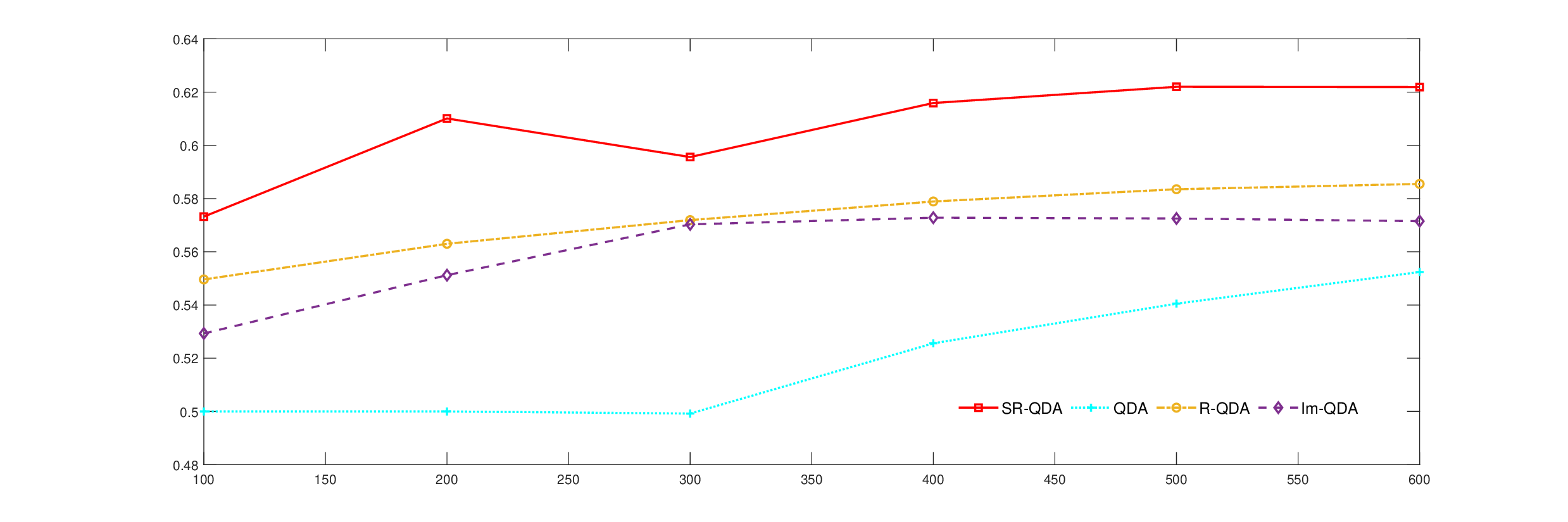}}
		\subfloat[\label{fig:b}$\sigma_1^2=1.5$]{\includegraphics[width=0.48\textwidth]{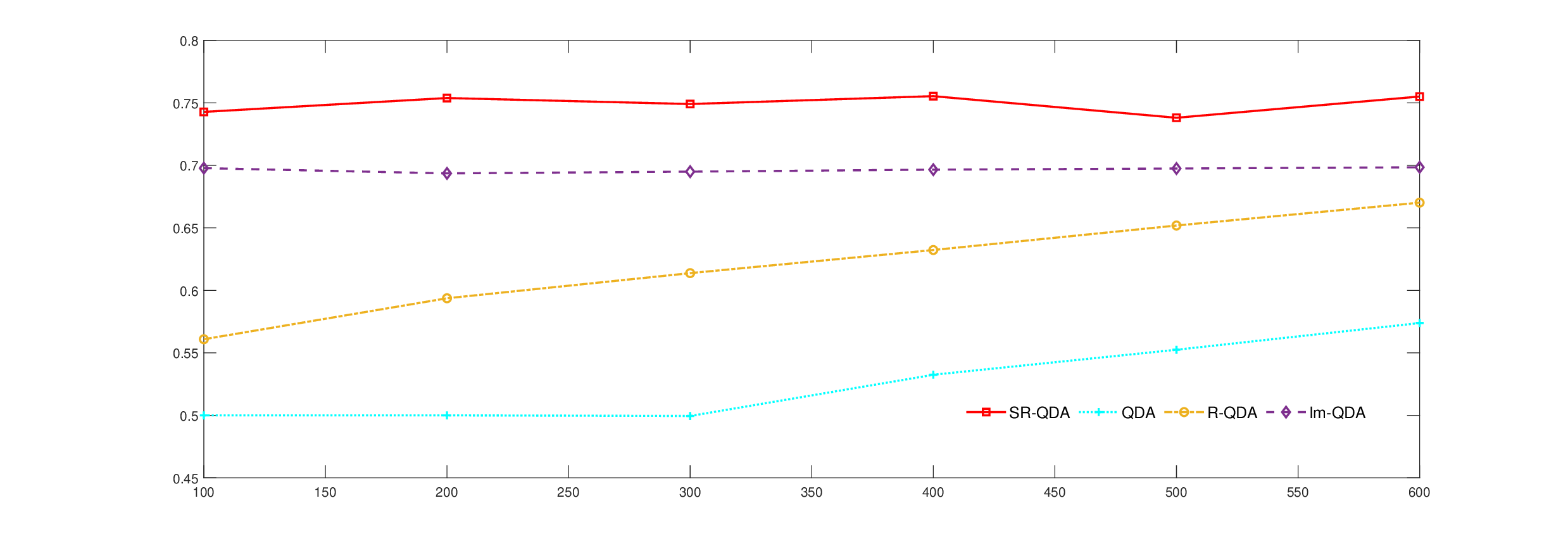}}\\
		\subfloat[\label{fig:c}$\sigma_1^2=2$]{\includegraphics[width=0.48\textwidth]{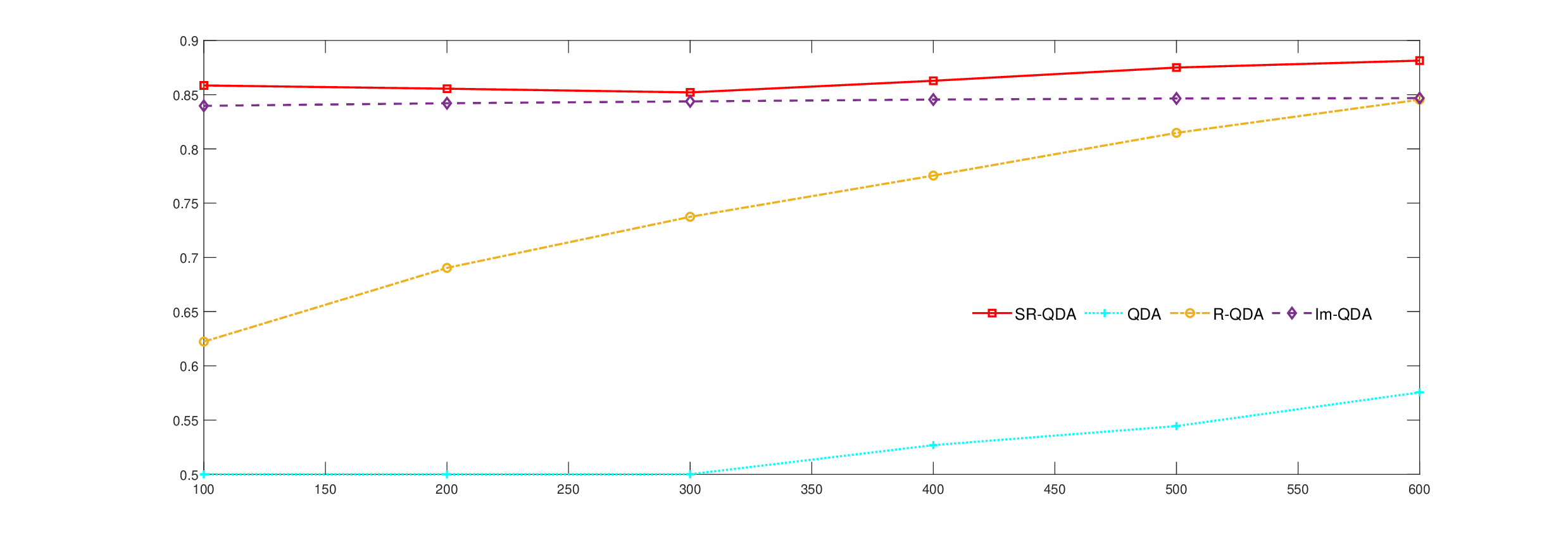}}
		\subfloat[\label{fig:d}$\sigma_1^2=4$]{\includegraphics[width=0.48\textwidth]{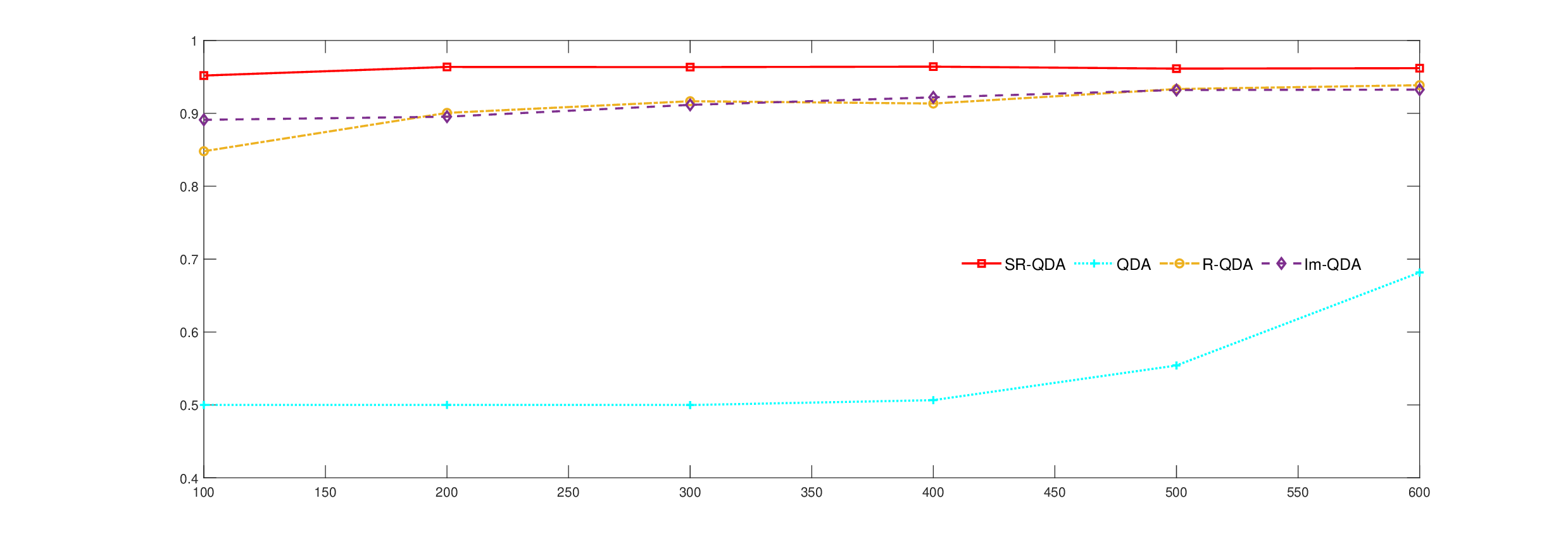}}
	\caption{Accuracy rate vs. training sample size $n$ for $p=150$ and $\pi_0=0.5$. Comparison for QDA, RLDA, ILDA and SRLDA with different values of $\sigma_1^2$.}	
\end{figure}

\subsection{Empirical analysis}

For empirical analysis, we use five real datasets, repeating each experiment 500 times, and compare the performance of all considered methods. For the first four real datasets, we rangdomly assigned $60\%$ of the observations to the training sample set and the remaining $40\%$ to the test sample set. For the fifth dataset, we investigate the performance of the proposed SR-QDA and other popular classifiers under different sample sizes. The five datasets are outlined below:

\begin{itemize}
	
	\item Dataset 1: Ultrasonic flowmeter diagnostics dataset. This dataset, provided by Gyamfi et al., is available at {\url{https://archive.ics.uci.edu/dataset/433/ultrasonic+flowmeter+diagnostics}.} It contains $87$ samples across two classes: "Healthy" and "Installation effects", with $p=36$ dimensions.

	\item Dataset 2: Ionosphere dataset. This dataset is collected by a system in Goose Bay, Labrador and is available at {\url{https://archive.ics.uci.edu/dataset/52/ionosphere}.} It includes $p=34$ attributes used to classfy radar into "good" and "bad" categories. There are $n_0=225$ samples for the "good" class and $n_1=126$ samples for the "bad" class.

	\item Dataset 3: Breast cancer Wisconsin dataset. This dataset is shared in UCI Machine Learning Repository and is specifically available at {\url{https://archive.ics.uci.edu/dataset/15/breast+cancer+wisconsin+original}.} There are $p=9$ features recorded for each people, excluding the class variable. The primary aim is to discriminate patients into two classes: "benign" and "malignant". It contains $n_0=458$ samples for the "benign" class and $n_1=241$ samples for the "malignant" class.

	\item Dataset 4: Handwritten digits dataset. This well-known dataset is available at {\url{https://yann.lecun.com/exdb/mnist/}}. It contains $70,000$ handwritten digits of 10 class labels, with $p=256$ dimensions. We use the highly confusing digits "4" and "9" to compare the performances of all menthods considered.

	\item Dataset 5: Epileptic seizure detection dataset. This dataset consists of recordings of brain activity using EEG signals and is publicly available at  {\url{https://www.kaggle.com/datasets/harunshimanto/epileptic-seizure-recognition/data?select=Epileptic+Seizure+Recognition.csv}}. It includes $5$ classes with $2300$ samples per class, each with $p=178$ dimensions. For our research, we focus on the most confusing classes for binary classification: class 4, which corresponds to recordings of patients with their eyes closed, and class 5, which corresponds to recordings of patients with their eyes open.
	
\end{itemize}

\begin{table}[]
	\centering
	\caption{Average accuracy rate for the binary classification of four datasets. Comparasion between the proposed method and other classifiers.}
	\vskip-0.3cm	
	\smallskip \small
	\label{tab:real_dataset1}
	\begin{threeparttable}
		\resizebox{0.9\textwidth}{!}{%
			\begin{tabular}{@{}cccccccccccc@{}}
				\toprule 
				\multirow{2}{*}{}                               & \multirow{2}{*}{QDA}    & \multirow{2}{*}{R-QDA}  & \multirow{2}{*}{Im-QDA} & \multirow{2}{*}{\textbf{SR-QDA}} & \multirow{2}{*}{SVM (gauss)} & \multirow{2}{*}{SVM (linear)} & \multirow{2}{*}{SVM (poly1)} & \multirow{2}{*}{SVM (rbf)} & \multirow{2}{*}{KNN1}   & \multirow{2}{*}{KNN3}   & \multirow{2}{*}{KNN5}   \\
				&                         &                         &                         &                                  &                              &                               &                              &                            &                         &                         &                         \\ \midrule
				\multicolumn{1}{c|}{\multirow{2}{*}{Dataset 1}} & \multirow{2}{*}{0.4283} & \multirow{2}{*}{0.6456} & \multirow{2}{*}{0.6678} & \multirow{2}{*}{\textbf{0.7529}} & \multirow{2}{*}{0.5917}      & \multirow{2}{*}{0.6000}       & \multirow{2}{*}{0.6000}      & \multirow{2}{*}{0.5908}    & \multirow{2}{*}{0.6853} & \multirow{2}{*}{0.6536} & \multirow{2}{*}{0.6793} \\
				\multicolumn{1}{c|}{}                           &                         &                         &                         &                                  &                              &                               &                              &                            &                         &                         &                         \\ \midrule
				\multicolumn{1}{c|}{\multirow{2}{*}{Dataset 2}} & \multirow{2}{*}{0.3617} & \multirow{2}{*}{0.8432} & \multirow{2}{*}{0.8223} & \multirow{2}{*}{\textbf{0.9283}} & \multirow{2}{*}{0.8425}      & \multirow{2}{*}{0.6383}       & \multirow{2}{*}{0.6383}      & \multirow{2}{*}{0.8425}    & \multirow{2}{*}{0.8328} & \multirow{2}{*}{0.8288} & \multirow{2}{*}{0.8322} \\
				\multicolumn{1}{c|}{}                           &                         &                         &                         &                                  &                              &                               &                              &                            &                         &                         &                         \\ \midrule
				\multicolumn{1}{c|}{\multirow{2}{*}{Dataset 3}} & \multirow{2}{*}{0.9513} & \multirow{2}{*}{0.9621} & \multirow{2}{*}{0.2272} & \multirow{2}{*}{\textbf{0.9666}} & \multirow{2}{*}{0.9682}      & \multirow{2}{*}{0.9657}       & \multirow{2}{*}{0.9597}      & \multirow{2}{*}{0.9682}    & \multirow{2}{*}{0.9550} & \multirow{2}{*}{0.9656} & \multirow{2}{*}{0.9675} \\
				\multicolumn{1}{c|}{}                           &                         &                         &                         &                                  &                              &                               &                              &                            &                         &                         &                         \\ \midrule
				\multicolumn{1}{c|}{\multirow{2}{*}{Dataset 4}} & \multirow{2}{*}{0.9699} & \multirow{2}{*}{0.9726} & \multirow{2}{*}{0.9225} & \multirow{2}{*}{\textbf{0.9782}} & \multirow{2}{*}{0.9738}      & \multirow{2}{*}{0.9728}       & \multirow{2}{*}{0.9574}      & \multirow{2}{*}{0.9738}    & \multirow{2}{*}{0.9807} & \multirow{2}{*}{0.9839} & \multirow{2}{*}{0.9817} \\
				\multicolumn{1}{c|}{}                           &                         &                         &                         &                                  &                              &                               &                              &                            &                         &                         &                         \\ \bottomrule
			\end{tabular}%
		}
		\begin{tablenotes}
			\footnotesize
			\item[-] QDA is the traditional QDA method; R-QDA is the regularized QDA method; \item[-] Im-QDA is the improved QDA method proposed by \cite{Sifaou2020HighDimensionalQD}; SR-QDA is the method we give;
			\item[-] SVM (gauss) and SVM (linear) are the support vector machines with gaussian kernel and linear kernel respectively; 
			\item[-] SVM (poly1) is the support vector machines with 1-order polynomial kernel;
			\item[-] SVM (rbf) is the support vector machines with radial basis function kernel;
			\item[-] KNN1, KNN3, and KNN5 are the k-nearest neighbor methods with 1, 3, and 5 neighbors, respectively. 
		\end{tablenotes}
	\end{threeparttable}
\end{table}
All experimental results based on the first four datesets are summarized and displayed in Table \ref{tab:real_dataset1}. In addition to our proposed SR-QDA method, we considered $10$ other classifiers in this experiment: QDA, R-QDA, Im-QDA proposed by \cite{Sifaou2020HighDimensionalQD}, SVM (gauss), SVM (linear), SVM (poly1), SVM (rbf), KNN1, KNN3 and KNN5. The results indicate that the QDA method fails for dataset $1$ and $2$, while other methods R-QDA, Im-QDA, SVMs and KNNs perform better. However, their classification accuracy is generally lower than that of our SR-QDA method. Therefore, SR-QDA outperforms the other methods in handling high-dimensional data with small sample sizes. On the other hand, for dataset $3$, which has low dimension and large sample, the SR-QDA classifier performs comparably to other widely recognized classification methods. Similar results were observed for dataset $4$ after dimensionality reduction using PCA.

In the final investigation, we use Dataset $5$: Epileptic seizure detection to study the performance of the proposed SR-QDA classifer in comparion with SVM and KNN deeply. For SVM, both linear and polynomial kernels are applied, while for KNN, the number of neighbors considered is $1$ and $3$. The accurary rates of these five methods are presented in Figure $\ref{fig:Realdata_Res2}$ versus the training sample size. The results demonstrate that SR-QDA outperforms all classifiers we considered, with other methods requiring more samples to achieve improved classfication accuracy. Even with a sample size as large as $2,000$, it is still insufficient, as the accuracy of KNN1, which performs best among these methods, remains lower than that of SR-QDA. 

\begin{figure}[H]
	\setlength{\abovecaptionskip}{0pt}
	\setlength{\belowcaptionskip}{0pt}
	\centering
	\includegraphics[height=5.51cm,width=10.56cm]{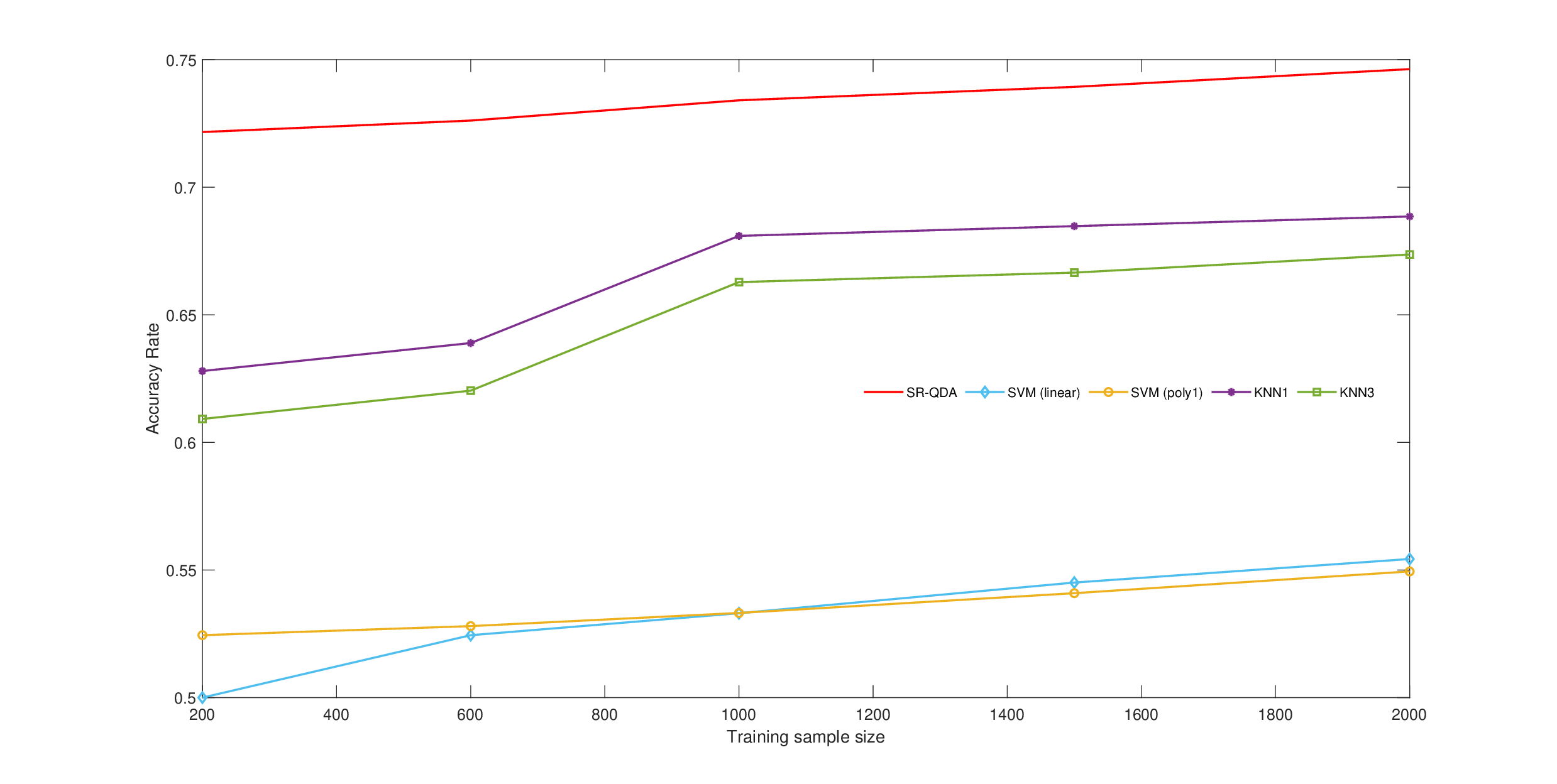}
	\caption{Accuracy rate vs. training sample size for $p=40$. Comparison between the proposed SR-QDA classifier, SVM and KNN using Epileptic Seizure Detection Database of class $4$ and $5$.}
	\label{fig:Realdata_Res2}
\end{figure}

%
%

{\appendix[]

Before presenting all proofs for theorems, we show the following lemma:

\begin{lemma}\label{lem:angle}
	Under Assumptions \ref{as:1} to \ref{as:4}, for any deterministic unit vector $\bom\in S_{\mathbb{R}}^{p-1}$, we have
	\beq\label{eq:angle}
	\left|\bom^T\fu_{j,i}\right|&\xrightarrow{a.s.}&\sqrt{\frac{\lambda_{j,i}^2-c_i}{\lambda_{j,i}\(\lambda_{j,i}+c_i\)}}\cdot\left|\bom^T\fv_{j,i}\right|
	\eeq
	where $j\in{\mathbb{I}_i}, i=0, 1$.
\end{lemma}
\begin{remark}
	{\color{blue}{This lemma is  well proofed in \cite{2022Spectrally}, which is omitted here.}}
\end{remark}

\begin{proof}[Proof of Theorem \ref{th1}]\label{proof_th1}
	Replacing $\widetilde{\fH}_i^{-1}$ by their expression shown in \eqref{H}, we can easily get
	$$Y_i(\widetilde{\fH}_0^{-1},\widetilde{\fH}_1^{-1})=\left(\frac{1}{\sigma_1^2}-\frac{1}{\sigma_0^2}\right)\fz^T\bSig_i\fz+\kappa_i+2\fy_i^T\fz-\xi_i,$$
	where $$\kappa_{i}= \sum_{j\in{\mathbb{I}_{2,0}}}\frac{\gamma^{(2)}_{j, 0}}{\sigma_{0}^{2}} \left(\fz^{T}\bSig_i^{\frac{1}{2}} \fu_{j, 0}\right)^{2}+\sum_{j\in{\mathbb{I}_{1,0}}}\frac{\gamma^{(1)}_{j,0}}{\sigma_0^2}\left(\fz^T\bSig_i^{\frac{1}{2}}\fu_{j,0}\right)^2-\sum_{j\in{\mathbb{I}_{2,1}}}\frac{\gamma^{(2)}_{j,1}}{\sigma_1^2}\left(\fz^T\bSig_i^{\frac{1}{2}}\fu_{j,1}\right)^2- \sum_{j\in{\mathbb{I}_{1,1}}} \frac{\gamma^{(1)}_{j, 1}}{\sigma_{1}^{2}}\left(\fz^{T}\bSig_i^{\frac{1}{2}} \fu_{j, 1}\right)^{2},$$ According to \cite{bai98}, we have
	\begin{equation}\label{baiQ}
	\frac{1}{p}z^T\bSig_iz-\frac{1}{p}tr(\bSig_i)\stackrel{a.s.}{\longrightarrow}0.
	\end{equation} Besides, the spiked model implies that $\frac{1}{p}tr(\bSig_i)\longrightarrow{\sigma_i^2}.$ Thus, $$\left(\frac{1}{\sigma_1^2}-\frac{1}{\sigma_0^2}\right)\fz^T\bSig_i\fz-p\sigma_i^2\left(\frac{1}{\sigma_1^2}-\frac{1}{\sigma_0^2}\right)\stackrel{a.s.}{\longrightarrow}0.$$
	Furthermore, we obtain that
	$$Y_i(\widetilde{\fH}_0^{-1},\widetilde{\fH}_1^{-1})-\left(p\sigma_i^2\left(\frac{1}{\sigma_1^2}-\frac{1}{\sigma_0^2}\right)+\kappa_i+2\fy_i^T\fz-\xi_i\right)\stackrel{a.s.}{\longrightarrow}0,$$
	which is the conclusion of Theorem \ref{th1}.
\end{proof}

\begin{proof}[Proof of Theorem \ref{th2}]\label{Proof_th2}
	First, we recall the following results showing in \cite{2011Random} that will used throughout the proof:
	\begin{equation}\label{limits}
	\begin{array}{r}
	\fu_{j, i}^{T} \fv_{k, i}\fv_{k, i}^{T} \fu_{j, i}-a_{j, i} \delta_{j, k} \stackrel{a . s .}{\longrightarrow} 0, \\
	\fu_{j, i}^{T} \fv_{k, \tilde{i}}\fv_{k, \tilde{i}}^{T} \fu_{k, i}-a_{j, i}\left(\fv_{j, i}^{T} \fv_{k, \tilde{i}}\right)^{2} \stackrel{a . s .}{\longrightarrow} 0, \\
	\frac{1}{\|\bmu\|^{2}} \bmu^{T} \fu_{j, i} \fu_{j, i}^{T} \bmu-a_{j, i} b_{j, i} \stackrel{a . s .}{\longrightarrow} 0,
	\end{array}
	\end{equation}
	where $\delta_{j,k}$ is Kronecker delta. We shall also recall the following formula allowing to compute the variance and covariance of quadratic forms of a multivariate normal distribution. If $\fz \sim \fN_p\left(\fzero, \fI\right)$ and $\fQ$ is a deterministic $p\times{p}$ matrix, then 
	\begin{equation}\label{eq:var}
	\var\left(\fz^{T} \fQ \fz\right)=2 \tr \fQ^{2}
	\end{equation}
	Let $\fQ_1$ and $\fQ_2$ be two deterministic $p\times{p}$ matrices, we similarly have:
	\begin{equation}\label{eq:cov}
	\cov\left(\fz^T\fQ_1\fz, \fz^T\fQ_2\fz\right)=2\tr \fQ_1\fQ_2.
	\end{equation}
	The expectation of $\widetilde{Y}_i$ is given by 
	$$m_i(\gamma)=p\sigma_i^2\left(\frac{1}{\sigma_1^2}-\frac{1}{\sigma_0^2}\right)+\tilde{\kappa}_i-\xi_i,$$ where
	$$\tilde{\kappa}_{i}=\mathbb{E}\kappa_i= \sum_{j\in{\mathbb{I}_{2,0}}}\frac{\gamma^{(2)}_{j, 0}}{\sigma_{0}^{2}} \fu_{j,0}^{T}\bSig_i \fu_{j, 0}+\sum_{j\in{\mathbb{I}_{1,0}}}\frac{\gamma^{(1)}_{j,0}}{\sigma_0^2}\fu_{j,0}^T\bSig_i\fu_{j,0}-\sum_{j\in{\mathbb{I}_{2,1}}}\frac{\gamma^{(2)}_{j,1}}{\sigma_1^2}\fu_{j,1}^T\bSig_i\fu_{j,1}- \sum_{j\in{\mathbb{I}_{1,1}}}\frac{\gamma^{(1)}_{j, 1}}{\sigma_{1}^{2}} \fu_{j,1}^{T} \bSig_i\fu_{j, 1}.$$
	
	According to the structure of $\bSig_i$ shown in \eqref{spikedmodel}, we have
	\begin{eqnarray}
	&\fu_{j,1}^T\bSig_i\fu_{j,1}
	=\sigma_i^2+\sum\limits_{\ell\in{\mathbb{I}_i}}\lambda_{\ell,i}\fu_{j,1}^T\fv_{\ell,i}\fv_{\ell,i}^T\fu_{j,1}, \\
	&\fu_{j,0}^T\bSig_i\fu_{j,0}
	=\sigma_i^2+\sum\limits_{\ell\in{\mathbb{I}_i}}\lambda_{\ell,i}\fu_{j,0}^T\fv_{\ell,i}\fv_{\ell,i}^T\fu_{j,0}.
	\end{eqnarray}
	Based on Lemma \ref{lem:angle}, we have
	\begin{equation}
	\fu_{j,0}^T\bSig_i\fu_{j,0}\stackrel{a.s.}{\longrightarrow}\left\{
	\begin{aligned}
	&\sigma_0^2\left[1+\lambda_{j,0}a_{j,0}\right], &\quad i=0;\\
	&\sigma_1^2\phi_{j,0}, &\quad i=1.
	\end{aligned}
	\right.
	\end{equation}
	\begin{equation}
	\fu_{j,1}^T\bSig_i\fu_{j,1}\stackrel{a.s.}{\longrightarrow}\left\{
	\begin{aligned}
	&	\sigma_0^2\phi_{j,1}, &\quad i=0;\\
	&	\sigma_1^2\left[1+\lambda_{j,1}a_{j,1}\right], &\quad i=1.
	\end{aligned}
	\right.
	\end{equation}
	
	Furthermore, it's easy to obtain
	\begin{eqnarray*}
		&\tilde{\kappa}_0-\left[\sum\limits_{j\in{\mathbb{I}_{2,0}}}\gamma^{(2)}_{j,0}(1+\lambda_{j,0}a_{j,0})+\sum\limits_{j\in{\mathbb{I}_{1,0}}}\gamma^{(1)}_{j,0}(1+\lambda_{j,0}a_{j,0})- \frac{\sigma_0^2}{\sigma_1^2}\sum\limits_{j\in{\mathbb{I}_{2,1}}}\gamma_{j,1}^{(2)}\phi_{j,1}-\frac{\sigma_0^2}{\sigma_1^2}\sum\limits_{j\in{\mathbb{I}_{1,1}}}\gamma_{j,1}^{(1)}\phi_{j,1}\right]\stackrel{a . s . }{\longrightarrow}0,
	\end{eqnarray*}
	\begin{eqnarray*}
		&\tilde{\kappa}_1-\left[\frac{\sigma_1^2}{\sigma_0^2} \sum\limits_{j\in{\mathbb{I}_{2,0}}}\gamma_{j,0}^{(2)}\phi_{j,0}+\frac{\sigma_1^2}{\sigma_0^2}\sum\limits_{j\in{\mathbb{I}_{1,0}}}\gamma_{j,0}^{(1)}\phi_{j,0}-\sum\limits_{j\in{\mathbb{I}_{2,1}}}\gamma_{j,1}^{(2)}(1+\lambda_{j,1}a_{j,1})-\sum\limits_{j\in{\mathbb{I}_{1,1}}}\gamma_{j,1}^{(1)}(1+\lambda_{j,1}a_{j,1})\right]
		\stackrel{a.s.}{\longrightarrow}0,
	\end{eqnarray*}
	As for $\xi_i=-2\eta^{SR-QDA}+(\bmu_i-\bar{\fx}_0)^T\widetilde{\fH}_0^{-1}(\bmu_i-\bar{\fx}_0)-(\bmu_i-\bar{\fx}_1)^T\widetilde{\fH}_1^{-1}(\bmu_i-\bar{\fx}_1)$, noting that $\bar{\fx}_i=\bmu_i+\frac{1}{n_i}\bSig_i^{\frac{1}{2}}\fZ_i\mathbf{1}_{n_i}$ and $\fZ_i\in{\mathbb{R}^{p\times{n_i}}}$ with entries i.i.d $\fN(0,1)$, we can rewrite
	\begin{eqnarray}\label{xi} 
	&\xi_0=-2\eta+\frac{1}{n_0^2}\mathbf{1}_{n_0}^T\fZ_0^T\bSig_0^{\frac{1}{2}}\widetilde{\fH}_0^{-1}\bSig_0^{\frac{1}{2}}\fZ_0\mathbf{1}_{n_0}-(\bmu+\frac{1}{n_1}\bSig_1^{\frac{1}{2}}\fZ_1\mathbf{1}_{n_1})^T\widetilde{\fH}_1^{-1}(\bmu+\frac{1}{n_1}\bSig_1^{\frac{1}{2}}\fZ_1\mathbf{1}_{n_1}),\\
	&\xi_1=-2\eta+(\bmu-\frac{1}{n_0}\bSig_0^{\frac{1}{2}}\fZ_0\mathbf{1}_{n_0})^T\widetilde{\fH}_0^{-1}(\bmu-\frac{1}{n_0}\bSig_0^{\frac{1}{2}}\fZ_0\mathbf{1}_{n_0})-\frac{1}{n_1^2}\mathbf{1}_{n_1}^T\fZ_1^T\bSig_1^{\frac{1}{2}}\widetilde{\fH}_1^{-1}\bSig_1^{\frac{1}{2}}\fZ_1\mathbf{1}_{n_1}.
	\end{eqnarray}
	According to Theorem 3.3.2 in \cite{0An}, the sample covariance matrix $\fS_i$ is independent of $\bar{\fx}_i$, which also means that $\bar{\fx}_i$ is independent of the eigenvectors of $\fS_i$ that appears in $\widetilde{\fH}_i^{-1}$. Therefore, we have
	\begin{equation}\label{xi:1}
	\frac{1}{n_i}\bmu^T\widetilde{\fH}_i^{-1}\bSig_i^{\frac{1}{2}}\fZ_i\mathbf{1}_{n_i}\stackrel{a.s.}{\longrightarrow}0,
	\end{equation}
	\begin{equation}
	\frac{1}{n_{i}^{2}} \mathbf{1}_{n_{i}}^{T} \fZ_{i}^{T}\bSig_i \widetilde{\fH}_{i}^{-1} \bSig_{i}\fZ_i \mathbf{1}_{n_{i}}-\frac{1}{n_{i}} \tr \bSig_{i} \widetilde{\fH}_{i}^{-1} \stackrel{a . s .}{\longrightarrow} 0.
	\end{equation}
	Combining the expressions of $\bSig_i$ and $\widetilde{\fH}_i$ and considering the fact that $r_i$ is finite, it's easy to have $\frac{1}{n_{i}} \tr \bSig_{i} \widetilde{\fH}_{i}^{-1}-c_i\stackrel{a . s .}{\longrightarrow} 0.$ Thus,
	\begin{equation}\label{xi:2}
	\frac{1}{n_{i}^{2}} \mathbf{1}_{n_{i}}^{T} \fZ_{i}^{T}\bSig_i \widetilde{\fH}_{i}^{-1} \bSig_{i}\fZ_i \mathbf{1}_{n_{i}}-c_i\stackrel{a . s .}{\longrightarrow} 0.
	\end{equation}
	Besides, replacing $\widetilde{\fH}_i^{-1}$ by its expression and applying \eqref{limits}, we can obtain 
	\begin{equation}\label{xi:3}
	\bmu^T\widetilde{\fH}_i^{-1}\bmu-\alpha_i\left(1-\sum\limits_{j\in{\mathbb{I}_{2,i}}}\gamma_{j,i}^{(2)}a_{j,i}b_{j,i}-\sum\limits_{j\in{\mathbb{I}_{1,i}}}\gamma_{j,i}^{(1)}a_{j,i}b_{j,i}\right)\stackrel{a . s .}{\longrightarrow} 0.
	\end{equation}
	Combining \eqref{xi}, \eqref{xi:1}, \eqref{xi:2} and \eqref{xi:3}, we get
	\begin{eqnarray}
	&\xi_{0}-\left[-2 \eta+c_{0}-c_{1}-\alpha_1\left(1-\sum\limits_{j\in{\mathbb{I}_{2,1}}}\gamma_{j,1}^{(2)}a_{j,1}b_{j,1}-\sum\limits_{j\in{\mathbb{I}_{1,1}}}\gamma_{j,1}^{(1)}a_{j,1}b_{j,1}\right)\right]\stackrel{a . s .}{\longrightarrow} 0.\\
	&\xi_{1}-\left[-2 \eta+c_{0}-c_{1}+\alpha_0\left(1-\sum\limits_{j\in{\mathbb{I}_{2,0}}}\gamma_{j,0}^{(2)}a_{j,0}b_{j,0}-\sum\limits_{j\in{\mathbb{I}_{1,0}}}\gamma_{j,0}^{(1)}a_{j,0}b_{j,0}\right)\right]\stackrel{a . s .}{\longrightarrow} 0.
	\end{eqnarray}
	So far we obtain the first convergengce result of Theorem \ref{th2}. Now, we address the convergence of $\vartheta_i(\gamma)$. Since $\fz$ is Gaussian, it's not hard to see that $\kappa_i$ and $\fy_i^T\fz$ are uncorrelated. Thus, we have
	\begin{equation}\label{var:total}
	\vartheta_i(\gamma)=\var(\kappa_i)+4\var(\fy_i^T\fz),
	\end{equation}
	Let's begin with $\var(\kappa_i)$ which can be written as
	\begin{equation}
	\begin{aligned}
	\var(\kappa_i)&=\frac{1}{\sigma_0^4}\sum\limits_{j\in{\mathbb{I}_{2,0}}}\left(\gamma_{j,0}^{(2)}\right)^2\var\left(\left[\fu_{j,0}^T\bSig_i^{\frac{1}{2}}\fz\right]^2\right)+\frac{2}{\sigma_0^4}\sum\limits_{j\neq\ell\in{\mathbb{I}_{2,0}}}\gamma_{j,0}^{(2)}\gamma_{\ell,0}^{(2)}\cov\left(\left[\fu_{j,0}^T\bSig_i^{\frac{1}{2}}\fz\right]^2, \left[\fu_{\ell,0}^T\bSig_i^{\frac{1}{2}}\fz\right]^2\right)\\
	&+\frac{1}{\sigma_0^4}\sum\limits_{j\in{\mathbb{I}_{1,0}}}\left(\gamma_{j,0}^{(1)}\right)^2\var\left(\left[\fu_{j,0}^T\bSig_i^{\frac{1}{2}}\fz\right]^2\right)+\frac{2}{\sigma_0^4}\sum\limits_{j\neq\ell\in{\mathbb{I}_{1,0}}}\gamma_{j,0}^{(1)}\gamma_{\ell,0}^{(1)}\cov\left(\left[\fu_{j,0}^T\bSig_i^{\frac{1}{2}}\fz\right]^2, \left[\fu_{\ell,0}^T\bSig_i^{\frac{1}{2}}\fz\right]^2\right)\\
	&+\frac{1}{\sigma_1^4}\sum\limits_{j\in{\mathbb{I}_{2,1}}}\left(\gamma_{j,1}^{(2)}\right)^2\var\left( \left[\fu_{j,1}^T\bSig_i^{\frac{1}{2}}\fz\right]^2\right)+\frac{2}{\sigma_1^4}\sum\limits_{j\neq\ell\in{\mathbb{I}_{2,1}}}\gamma_{j,1}^{(2)}\gamma_{\ell,1}^{(2)}\cov\left(\left[\fu_{j,1}^T\bSig_i^{\frac{1}{2}}\fz\right]^2, \left[\fu_{\ell,1}^T\bSig_i^{\frac{1}{2}}\fz\right]^2\right)\\
	&+\frac{1}{\sigma_1^4}\sum\limits_{j\in{\mathbb{I}_{1,1}}}\left(\gamma_{j,1}^{(1)}\right)^2\var\left( \left[\fu_{j,1}^T\bSig_i^{\frac{1}{2}}\fz\right]^2\right)+\frac{2}{\sigma_1^4}\sum\limits_{j\neq\ell\in{\mathbb{I}_{1,1}}}\gamma_{j,1}^{(1)}\gamma_{\ell,1}^{(1)}\cov\left(\left[\fu_{j,1}^T\bSig_i^{\frac{1}{2}}\fz\right]^2, \left[\fu_{\ell,1}^T\bSig_i^{\frac{1}{2}}\fz\right]^2\right)\\
	&+\sum\limits_{j\in{\mathbb{I}_{2,0}}}\sum\limits_{\ell\in{\mathbb{I}_{1,0}}}\frac{2\gamma_{j,0}^{(2)}\gamma_{\ell,0}^{(1)}}{\sigma_0^4}\cov\left(\left[\fu_{j,0}^T\bSig_i^{\frac{1}{2}}\fz\right]^2, \left[\fu_{\ell,0}^T\bSig_i^{\frac{1}{2}}\fz\right]^2\right)\\
	&-\sum\limits_{j\in{\mathbb{I}_{2,0}}}\sum\limits_{\ell\in{\mathbb{I}_{2,1}}}\frac{2\gamma_{j,0}^{(2)}\gamma_{\ell,1}^{(2)}}{\sigma_1^2\sigma_0^2}\cov\left(\left[\fu_{j,0}^T\bSig_i^{\frac{1}{2}}\fz\right]^2, \left[\fu_{\ell,1}^T\bSig_i^{\frac{1}{2}}\fz\right]^2\right)\\
	&-\sum\limits_{j\in{\mathbb{I}_{2,0}}}\sum\limits_{\ell\in{\mathbb{I}_{1,1}}}\frac{2\gamma_{j,0}^{(2)}\gamma_{\ell,1}^{(1)}}{\sigma_1^2\sigma_0^2}\cov\left(\left[\fu_{j,0}^T\bSig_i^{\frac{1}{2}}\fz\right]^2, \left[\fu_{\ell,1}^T\bSig_i^{\frac{1}{2}}\fz\right]^2\right)\\
	&-\sum\limits_{j\in{\mathbb{I}_{1,0}}}\sum\limits_{\ell\in{\mathbb{I}_{2,1}}}\frac{2\gamma_{j,0}^{(1)}\gamma_{\ell,1}^{(2)}}{\sigma_1^2\sigma_0^2}\cov\left(\left[\fu_{j,0}^T\bSig_i^{\frac{1}{2}}\fz\right]^2, \left[\fu_{\ell,1}^T\bSig_i^{\frac{1}{2}}\fz\right]^2\right)\\
	&-\sum\limits_{j\in{\mathbb{I}_{1,0}}}\sum\limits_{\ell\in{\mathbb{I}_{1,1}}}\frac{2\gamma_{j,0}^{(1)}\gamma_{\ell,1}^{(1)}}{\sigma_1^2\sigma_0^2}\cov\left(\left[\fu_{j,0}^T\bSig_i^{\frac{1}{2}}\fz\right]^2, \left[\fu_{\ell,1}^T\bSig_i^{\frac{1}{2}}\fz\right]^2\right)\\
	&+\sum\limits_{j\in{\mathbb{I}_{2,1}}}\sum\limits_{\ell\in{\mathbb{I}_{1,1}}}\frac{2\gamma_{j,1}^{(2)}\gamma_{\ell,1}^{(2)}}{\sigma_1^4}\cov\left(\left[\fu_{j,1}^T\bSig_i^{\frac{1}{2}}\fz\right]^2, \left[\fu_{\ell,1}^T\bSig_i^{\frac{1}{2}}\fz\right]^2\right).
	\end{aligned}
	\end{equation}
	Among them, based on \eqref{eq:var}, replacing $\bSig_i$ by its expression and applying \eqref{limits}, it's easy to get 
	\begin{equation}\label{varv0:1}
	\var\left(\left[\fu_{j,0}^T\bSig_i^{\frac{1}{2}}\fz\right]^2\right)\stackrel{a . s .}{\longrightarrow} \left\{
	\begin{aligned}
	&2\sigma_0^4(1+\lambda_{j,0}a_{j,0})^2, &\quad i=0;\\
	&2\sigma_1^4\phi_{j,0}^2, & \quad i=1.
	\end{aligned}\right. 
	\end{equation}
	and 
	\begin{equation}\label{varv0:2}
	\var\left(\left[\fu_{j,1}^T\bSig_i^{\frac{1}{2}}\fz\right]^2\right)\stackrel{a . s .}{\longrightarrow} \left\{
	\begin{aligned}
	&2\sigma_0^4\phi_{j,1}^2, &\quad i=0;\\
	&2\sigma_1^4(1+\lambda_{j,1}a_{j,1})^2, & \quad i=1.
	\end{aligned}\right.
	\end{equation}
	
	Besides, for $j\neq\ell$, according to \eqref{limits}, we get
	\begin{equation}\label{covlim1}
	\cov\left(\left[\fu_{j,0}^T\bSig_i^{\frac{1}{2}}\fz\right]^2, \left[\fu_{\ell,0}^T\bSig_i^{\frac{1}{2}}\fz\right]^2\right)=
	2\left(\fu_{j,0}^T\bSig_i\fu_{\ell,0}\right)^2\stackrel{a.s.}{\longrightarrow}\left\{
	\begin{aligned}
	&\quad0, &i=0;\\
	&2\sigma_1^4\theta_{j,\ell,1}^2, &i=1.
	\end{aligned}
	\right.
	\end{equation}
	\begin{equation}\label{covlim2}
	\cov\left(\left[\fu_{j,1}^T\bSig_i^{\frac{1}{2}}\fz\right]^2, \left[\fu_{\ell,1}^T\bSig_i^{\frac{1}{2}}\fz\right]^2\right)=
	2\left(\fu_{j,1}^T\bSig_i\fu_{\ell,1}\right)^2\stackrel{a.s.}{\longrightarrow}
	\left\{
	\begin{aligned}
	&2\sigma_0^4\theta_{j,\ell,0}^2, &i=0;\\
	&0, &i=1.
	\end{aligned}
	\right.
	\end{equation}
	and 
	\begin{equation}\label{covlim3}
	\cov\left(\left[\fu_{j,0}^T\bSig_i^{\frac{1}{2}}\fz\right]^2, \left[\fu_{\ell,1}^T\bSig_i^{\frac{1}{2}}\fz\right]^2\right)=
	2\left(\fu_{j,0}^T\bSig_i\fu_{\ell,1}\right)^2\stackrel{a.s.}{\longrightarrow}
	\left\{
	\begin{aligned}
	&2\sigma_0^4(1+\lambda_{j,0})^2a_{j,0}a_{\ell,1}\psi_{j,\ell,0,1}^2, &i=0;\\
	&2\sigma_1^4(1+\lambda_{\ell,1})^2a_{j,0}a_{\ell,1}\psi_{j,\ell,0,1}^2, &i=1.
	\end{aligned}
	\right.
	\end{equation}

	Combining \eqref{varv0:1}, \eqref{varv0:2}, \eqref{covlim1}, \eqref{covlim2} and \eqref{covlim3}, we obtain
	
	\begin{equation}\label{vark:lim}
	\var(\kappa_i)-\bar{\varrho}_i(\gamma)\stackrel{a . s .}{\longrightarrow} 0,\quad 
	\end{equation}
	
	\begin{equation*}
	\begin{aligned}
	\bar{\varrho}_i(\gamma)&=2\sum\limits_{j\in{\mathbb{I}_{2,0}}}\left(\gamma_{j,0}^{(2)}\right)^2\left[\tilde{i}(1+\lambda_{j,0}a_{j,0})+i\frac{\sigma_1^4}{\sigma_0^4}\phi_{j,0}^2\right]+2\sum\limits_{j\in{\mathbb{I}_{1,0}}}\left(\gamma_{j,0}^{(1)}\right)^2\left[\tilde{i}(1+\lambda_{j,0}a_{j,0})+i\frac{\sigma_1^4}{\sigma_0^4}\phi_{j,0}^2\right]\\
	&+2\sum\limits_{j\in{\mathbb{I}_{2,1}}}\left(\gamma_{j,1}^{(2)}\right)^2\left[\tilde{i}\frac{\sigma_0^4}{\sigma_1^4}\phi_{j,1}^2+i(1+\lambda_{j,1}a_{j,1})^2\right]+2\sum\limits_{j\in{\mathbb{I}_{1,1}}}\left(\gamma_{j,1}^{(1)}\right)^2\left[\tilde{i}\frac{\sigma_0^4}{\sigma_1^4}\phi_{j,1}^2+i(1+\lambda_{j,1}a_{j,1})^2\right]\\
	&+4\frac{\sigma_i^4}{\sigma_{\tilde{i}}^4}\sum\limits_{j\neq{\ell}\in{\mathbb{I}_{2,\tilde{i}}}}\gamma_{j,\tilde{i}}^{(2)}\gamma_{\ell,\tilde{i}}^{(2)}\theta_{j,\ell,i}^2+4\frac{\sigma_i^4}{\sigma_{\tilde{i}}^4}\sum\limits_{j\neq{\ell}\in{\mathbb{I}_{1,\tilde{i}}}}\gamma_{j,\tilde{i}}^{(1)}\gamma_{\ell,\tilde{i}}^{(1)}\theta_{j,\ell,i}^2+4\frac{\sigma_i^4}{\sigma_{\tilde{i}}^4}\sum\limits_{j\in{\mathbb{I}_{2,\tilde{i}}}}\sum\limits_{\ell\in{\mathbb{I}_{1,\tilde{i}}}}\gamma_{j,\tilde{i}}^{(2)}\gamma_{\ell,\tilde{i}}^{(1)}\theta_{j,\ell,i}^2\\
	&-4\frac{\sigma_i^2}{\sigma_{\tilde{i}}^2}\sum\limits_{j\in{\mathbb{I}_{2,0}}}\sum\limits_{\ell\in{\mathbb{I}_{2,1}}}\gamma_{j,0}^{(2)}\gamma_{\ell,1}^{(2)}a_{j,0}a_{\ell,1}(1+\tilde{i}\lambda_{j,0}+i\lambda_{\ell,1})^2\psi_{j,\ell,0,1}^2\\
	&-4\frac{\sigma_1^2}{\sigma_0^2}\sum\limits_{j\in{\mathbb{I}_{2,0}}}\sum\limits_{\ell\in{\mathbb{I}_{1,1}}}\gamma_{j,0}^{(2)}\gamma_{\ell,1}^{(1)}a_{j,0}a_{\ell,1}(1+\tilde{i}\lambda_{j,0}+i\lambda_{\ell,1})^2\psi_{j,\ell,0,1}^2\\
	&-4\frac{\sigma_1^2}{\sigma_0^2}\sum\limits_{j\in{\mathbb{I}_{1,0}}}\sum\limits_{\ell\in{\mathbb{I}_{2,1}}}\gamma_{j,0}^{(1)}\gamma_{\ell,1}^{(2)}a_{j,0}a_{\ell,1}(1+\tilde{i}\lambda_{j,0}+i\lambda_{\ell,1})^2\psi_{j,\ell,0,1}^2\\
	&-4\frac{\sigma_1^2}{\sigma_0^2}\sum\limits_{j\in{\mathbb{I}_{1,0}}}\sum\limits_{\ell\in{\mathbb{I}_{1,1}}}\gamma_{j,0}^{(1)}\gamma_{\ell,1}^{(1)}a_{j,0}a_{\ell,1}(1+\tilde{i}\lambda_{j,0}+i\lambda_{\ell,1})^2\psi_{j,\ell,0,1}^2.
	\end{aligned}
	\end{equation*}

	Now let's deal with $\var(\fy_i^T\fz)$, which can be written as follows:
	\begin{equation*}
	\begin{aligned}
	&\var(\fy_i^T\fz)=\mathbb{E}\fy_i^T\fz\fz^T\fy_i=\fy_i^T\fy_i\\
	&= \left(\widetilde{\fH}_1^{-1}(\tilde{\bmu}_{i,1} -\frac{\bSig_1^{\frac{1}{2}}\fZ_1\mathbf{1}_{n1}}{n_1})-\widetilde{\fH}_{0}^{-1}(\tilde{\bmu}_{i,0}-\frac{\bSig_0^{\frac{1}{2}}\fZ_0\mathbf{1}_{n_0}}{n_0} )\right)^{T} \bSig_{i} \left(\widetilde{\fH}_1^{-1}(\tilde{\bmu}_{i,1}-\frac{\bSig_1^{\frac{1}{2}}\fZ_1\mathbf{1}_{n1}}{n_1})-\widetilde{\fH}_{0}^{-1}(\tilde{\bmu}_{i,0}-\frac{\bSig_0^{\frac{1}{2}}\fZ_0\mathbf{1}_{n_0}}{n_0} )\right),
	\end{aligned}
	\end{equation*}
	where $\tilde{\bmu}_{i,1}=\bmu_i-\bmu_1$, $\tilde{\bmu}_{i,0}=\bmu_i-\bmu_0$, for special case $\tilde{\bmu}_{1,1}=\tilde{\bmu}_{0,0}=0$ and $\tilde{\bmu}_{0,1}=-\bmu$, $\tilde{\bmu}_{1,0}=\bmu$.
	With the same arguments for \eqref{xi:1}, the following convergence holds:
	\begin{equation*}\label{var:y1}
	\begin{aligned}
	&\frac{1}{n_1}\bmu^T\widetilde{\fH}_1^{-1}\bSig_i\widetilde{\fH}_1^{-1}\bSig_1^{\frac{1}{2}}\fZ_1\mathbf{1}_{n_1}\stackrel{a . s .}{\longrightarrow} 0,\quad \frac{1}{n_0}\bmu^T\widetilde{\fH}_1^{-1}\bSig_i\widetilde{\fH}_0^{-1}\bSig_0^{\frac{1}{2}}\fZ_0\mathbf{1}_{n_0}\stackrel{a . s .}{\longrightarrow} 0, \quad i=0, 1;\\
	&\frac{1}{n_1}\bmu^T\widetilde{\fH}_0^{-1}\bSig_i\widetilde{\fH}_1^{-1}\bSig_1^{\frac{1}{2}}\fZ_1\mathbf{1}_{n_1}\stackrel{a . s .}{\longrightarrow} 0,\quad \frac{1}{n_0}\bmu^T\widetilde{\fH}_0^{-1}\bSig_i\widetilde{\fH}_0^{-1}\bSig_0^{\frac{1}{2}}\fZ_0\mathbf{1}_{n_0}\stackrel{a . s .}{\longrightarrow} 0, \quad i=0, 1.
	\end{aligned}
	\end{equation*}
	Since $\bSig_{0}^{\frac{1}{2}}\fZ_0$ is independent of  $\bSig_1^{\frac{1}{2}}\fZ_1$,
	\begin{equation*}\label{var:y2}
	\frac{1}{n_0n_1} \mathbf{1}_{n_{0}}^{T} \fZ_{0}^{T}\bSig_0^{\frac{1}{2}} \widetilde{\fH}_{0}^{-1}\bSig_i \widetilde{\fH}_{1}^{-1}\bSig_{1}^{\frac{1}{2}}\fZ_{1}  \mathbf{1}_{n_{1}} \stackrel{\text { a.s. }}{\longrightarrow} 0, \quad i=0,1.
	\end{equation*}
	Therefore,
	\begin{equation}
	\var(\fy_i^T\fz)=\bmu^T\widetilde{\fH}_{\tilde{i}}^{-1}\bSig_i\widetilde{\fH}_{\tilde{i}}^{-1}\bmu+\frac{1}{n_{1}^{2}} \mathbf{1}_{n_{1}}^{T} \fZ_1^T\bSig_1^{\frac{1}{2}} \widetilde{\fH}_{1}^{-1} \bSig_{i}\widetilde{\fH}_{1}^{-1}\bSig_1^{\frac{1}{2}}\fZ_1 \mathbf{1}_{n_{1}}+\frac{1}{n_{0}^{2}} \mathbf{1}_{n_{0}}^{T} \fZ_0^T\bSig_0^{\frac{1}{2}} \widetilde{\fH}_{0}^{-1} \bSig_{i}\widetilde{\fH}_{0}^{-1}\bSig_0^{\frac{1}{2}}\fZ_0 \mathbf{1}_{n_{0}}+o_{\text{a.s.}}(1).
	\end{equation}
	While from Theorem 3.4 in \cite{2011Random} , we can derive:
	\begin{eqnarray}
	&\frac{1}{n_{1}^{2}} \mathbf{1}_{n_{1}}^{T} \fZ_1^T\bSig_1^{\frac{1}{2}} \widetilde{\fH}_{1}^{-1} \bSig_{i}\widetilde{\fH}_{1}^{-1}\bSig_1^{\frac{1}{2}}\fZ_1 \mathbf{1}_{n_{1}}-\frac{1}{n_{1}} \tr \bSig_{1} \widetilde{\fH}_1^{-1} \bSig_{i} \widetilde{\fH}_{1}^{-1} \stackrel{\text { a.s. }}{\longrightarrow} 0,\\
	&\frac{1}{n_{0}^{2}} \mathbf{1}_{n_{0}}^{T} \fZ_0^T\bSig_0^{\frac{1}{2}} \widetilde{\fH}_{0}^{-1} \bSig_{i}\widetilde{\fH}_{0}^{-1}\bSig_0^{\frac{1}{2}}\fZ_0 \mathbf{1}_{n_{0}}-\frac{1}{n_{0}} \tr \bSig_{0} \widetilde{\fH}_0^{-1} \bSig_{i} \widetilde{\fH}_{0}^{-1} \stackrel{\text { a.s. }}{\longrightarrow} 0,
	\end{eqnarray}
	Replacing $\bSig_i$ and $\widetilde{\fH}_i^{-1}$ by their expressions with the fact that $r_i$ is finite, we get 
	\begin{eqnarray}\label{vary:1}
	&\frac{1}{n_{1}^{2}} \mathbf{1}_{n_{1}}^{T} \fZ_1^T\bSig_1^{\frac{1}{2}} \widetilde{\fH}_{1}^{-1} \bSig_{i}\widetilde{\fH}_{1}^{-1}\bSig_1^{\frac{1}{2}}\fZ_1 \mathbf{1}_{n_{1}}-c_1\frac{\sigma_i^2}{\sigma_1^2}\stackrel{\text { a.s. }}{\longrightarrow} 0,\\
	&\frac{1}{n_{0}^{2}} \mathbf{1}_{n_{0}}^{T} \fZ_0^T\bSig_0^{\frac{1}{2}} \widetilde{\fH}_{0}^{-1} \bSig_{i}\widetilde{\fH}_{0}^{-1}\bSig_0^{\frac{1}{2}}\fZ_0 \mathbf{1}_{n_{0}}-c_0\frac{\sigma_i^2}{\sigma_0^2} \stackrel{\text { a.s. }}{\longrightarrow} 0,
	\end{eqnarray}
	
	Next, we deal with $\bmu^T\widetilde{\fH}_{\tilde{i}}^{-1}\bSig_i\widetilde{\fH}_{\tilde{i}}^{-1}\bmu$. According to \eqref{limits}, after some routine calculations, we get
	
	\begin{equation}\label{vary:2}
	\bmu^T\widetilde{\fH}_{\tilde{i}}^{-1}\bSig_i\widetilde{\fH}_{\tilde{i}}^{-1}\bmu-\varpi_i\stackrel{\text { a.s. }}{\longrightarrow} 0,
	\end{equation}
	where
	\begin{equation*}
	\begin{aligned}
	\varpi_i=&\alpha_{\tilde{i}}\frac{\sigma_i^2}{\sigma_{\tilde{i}}^2}\left[1+\sum\limits_{k\in{\mathbb{I}_i}}\lambda_{k,i}b_{k,i}-2\sum\limits_{j\in{\mathbb{I}_{2,\tilde{i}}}}\gamma_{j,\tilde{i}}^{(2)}a_{j,\tilde{i}}b_{j,\tilde{i}}-2\sum\limits_{j\in{\mathbb{I}_{1,\tilde{i}}}}\gamma_{j,\tilde{i}}^{(1)}a_{j,\tilde{i}}b_{j,\tilde{i}}\right.\\
	&-2\sum\limits_{j\in{\mathbb{I}_{2,\tilde{i}}}}\sum\limits_{k\in{\mathbb{I}_i}}\gamma_{j,\tilde{i}}^{(2)}\lambda_{k,i}a_{j,\tilde{i}}\sqrt{b_{j,\tilde{i}}b_{k,i}\psi_{j,k,\tilde{i},i}^2}-2\sum\limits_{j\in{\mathbb{I}_{1,\tilde{i}}}}\sum\limits_{k\in{\mathbb{I}_i}}\gamma_{j,\tilde{i}}^{(1)}\lambda_{k,i}a_{j,\tilde{i}}\sqrt{b_{j,\tilde{i}}b_{k,i}\psi_{j,k,\tilde{i},i}^2}\\
	&+\sum\limits_{j\in{\mathbb{I}_{2,\tilde{i}}}}\left(\gamma_{j,\tilde{i}}^{(2)}\right)^2a_{j,\tilde{i}}b_{j,\tilde{i}}+\sum\limits_{j\in{\mathbb{I}_{1,\tilde{i}}}}\left(\gamma_{j,\tilde{i}}^{(1)}\right)^2a_{j,\tilde{i}}b_{j,\tilde{i}}\\
	&+\sum\limits_{j,\ell\in{\mathbb{I}_{2,\tilde{i}}}}\sum\limits_{k\in{\mathbb{I}_i}}\gamma_{j,\tilde{i}}^{(2)}\gamma_{\ell,\tilde{i}}^{(2)}\lambda_{k,i}a_{j,\tilde{i}}a_{\ell,\tilde{i}}\sqrt{b_{j,\tilde{i}}b_{\ell,\tilde{i}}\psi_{j,k,\tilde{i},i}^2\psi_{\ell,k,\tilde{i},i}^2}\\
	&+\sum\limits_{j,\ell\in{\mathbb{I}_{1,\tilde{i}}}}\sum\limits_{k\in{\mathbb{I}_i}}\gamma_{j,\tilde{i}}^{(1)}\gamma_{\ell,\tilde{i}}^{(1)}\lambda_{k,i}a_{j,\tilde{i}}a_{\ell,\tilde{i}}\sqrt{b_{j,\tilde{i}}b_{\ell,\tilde{i}}\psi_{j,k,\tilde{i},i}^2\psi_{\ell,k,\tilde{i},i}^2}\\
	&\left.
	+2\sum\limits_{j\in{\mathbb{I}_{2,\tilde{i}}}}\sum\limits_{\ell\in{\mathbb{I}_{1,\tilde{i}}}}\sum\limits_{k\in{\mathbb{I}_i}}\gamma_{j,\tilde{i}}^{(2)}\gamma_{\ell,\tilde{i}}^{(1)}\lambda_{k,i}a_{j,\tilde{i}}a_{\ell,\tilde{i}}\sqrt{b_{j,\tilde{i}}b_{\ell,\tilde{i}}\psi_{j,k,\tilde{i},i}^2\psi_{\ell,k,\tilde{i},i}^2}
	\right],
	\end{aligned}
	\end{equation*}
	Therefore, according to \eqref{vary:1} and \eqref{vary:2}, we obtain
	\begin{equation}\label{vary}
	\var(y_i^Tz)-\left(c_1\frac{\sigma_i^2}{\sigma_1^2}+c_0\frac{\sigma_i^2}{\sigma_0^2}+\varpi_i\right)\stackrel{\text { a.s. }}{\longrightarrow} 0	.
	\end{equation}
	What's more, combining with \eqref{var:total} and \eqref{vark:lim}, we get the convergence of $\vartheta_i(\gamma)$ as below: $$\vartheta_0(\gamma)-\bar{\vartheta}_0(\gamma)\stackrel{\text { a.s. }}{\longrightarrow} 0,\quad \vartheta_1(\gamma)-\bar{\vartheta}_1(\gamma)\stackrel{\text { a.s. }}{\longrightarrow} 0,$$
	where $\bar{\vartheta}_0(\gamma)=\bar{\varrho}_0(\gamma)+4\left(c_1\frac{\sigma_0^2}{\sigma_1^2}+c_0+\varpi_0\right)$, $\bar{\vartheta}_1(\gamma)=\bar{\varrho}_1(\gamma)+4\left(c_1+c_0\frac{\sigma_1^2}{\sigma_0^2}+\varpi_0\right)$.
	Putting all these results together and rewriting them in vector form yields the convergence of the variance $\vartheta_0(\gamma)$.
\end{proof}

\begin{proof}[Proof of Theorem \ref{th4}]\label{proof_th4}
	Let $\omega_{1,i}=\frac{\gamma_{1,i}\lambda_{1,i}}{1+\lambda_{1,i}\gamma_{1,i}}$, $\omega_{2,i}=\frac{-\gamma_{2,i}\lambda_{-1,i}}{1-\lambda_{-1,i}\gamma_{2,i}}$, $i\in{\{0,1\}}$, easy to verify that $\omega_{1,i}, \omega_{2,i}\in{(0,1)}$. Then $\gamma_{1,i}$ and $\gamma_{2,i}$ can be rewritten as $$\gamma_{1,i}=\frac{\omega_{1,i}}{\lambda_{1,i}(1-\omega_{1,i})}\quad \text{and} \quad \gamma_{2,i}=\frac{-\omega_{2,i}}{\lambda_{-1,0}(1-\omega_{2,i})}.$$ Further,
	\begin{equation}\label{w:ga1}
	\gamma_{j,i}^{(1)}=\frac{\omega_{1,i}\lambda_{j,i}}{\lambda_{1,i}(1-\omega_{1,i})+\omega_{1,i}\lambda_{j,i}}, j\in{\mathbb{I}_{1,i}}
	\end{equation}
	\begin{equation}\label{w:ga2}
	\gamma_{j,i}^{(2)}=\frac{-\omega_{2,i}\lambda_{j,i}}{\lambda_{-1,i}(1-\omega_{2,i})-\omega_{2,i}\lambda_{j,i}}, j\in{\mathbb{I}_{2,i}} 
	\end{equation}
	Plugging \eqref{w:ga1} and \eqref{w:ga2} into \eqref{barm}and \eqref{barV}, we have $\widetilde{m}(\omega)=\overline{m}(\gamma)$ and $\widetilde{\vartheta}(\omega)=\overline{m}(\gamma).$
	Thus the objection function can be written as 
	$$\widetilde{\rho}(\omega)=\frac{|\widetilde{m}_0(\omega)-\widetilde{m}_1(\omega)|}{\sqrt{\widetilde{\vartheta}_0(\omega)+\widetilde{\vartheta}_1(\omega)}},$$
	therefore, \eqref{obj:ga} is equivalent to 
	$$\omega_{1,0}^*, \omega_{2,0}^*, \omega_{1,1}^*, \omega_{2,1}^*={\arg\max}_{\omega_{1,0},\omega_{2,0},\omega_{1,1},\omega_{2,1}\in{(0,1)}}\widetilde{\rho}(\omega_{1,0},\omega_{2,0},\omega_{1,1},\omega_{2,1}),$$
	by which
	$$\gamma_{2,0}^*=\frac{-\omega_{2,0}^*}{\lambda_{-1,0}(1-\omega_{2,0}^*)} \quad \text{and} \quad \gamma_{1,0}^*=\frac{\omega_{1,0}^*}{\lambda_{1,0}(1-\omega_{1,0}^*)},$$
	$$\gamma_{2,1}^*=\frac{-\omega_{2,1}^*}{\lambda_{-1,1}(1-\omega_{2,1}^*)} \quad \text{and} \quad \gamma_{1,1}^*=\frac{\omega_{1,1}^*}{\lambda_{1,1}(1-\omega_{1,1}^*)}.$$
\end{proof}

\begin{proof}[Proof of Theorem \ref{thm:5}]\label{proof_th5}
	Since $\hat{\sigma}_i^{2}=\frac{1}{p-(r_{1,i}+r_{2,i})}\left(\sum_{j=1}^{p}l_{j,i}-\sum_{j=1}^{r_{1,i}}l_{j,i}-\sum_{j=-r_{2,i}}^{-1}l_{j,i} \right), $ then $ (p-(r_{1,i}+r_{2,i}))\hat{\sigma}_i^{2}=\sum_{j=1}^{p}l_{j,i}-\sum_{j=1}^{r_{1,i}}l_{j,i}-\sum_{j=-r_{2,i}}^{-1}l_{j,i}. $ For the first term $ \sum_{j=1}^{p}l_{j,i}, $ by \cite{bai04}, we have 
	\begin{align*}
	\sum_{j=1}^{p}l_{j,i}-p\int xdF^{J_{i},H_{n}}\stackrel{\mathcal{D}}{\rightarrow} \mathcal{N}\left(0, 2J_i\sigma^{4}\right).
	\end{align*}
	Furthermore, by \cite{Bai10},
	\begin{align*}
	&p\int xdF^{J_{i},H_{n}}=p\int tdH_{n}(t)=(p-(r_{1,i}+r_{2,i}))\sigma_i^{2}+\sigma_i^{2}\sum_{j=1}^{r_{1,i}}\left(\frac{\lambda_{j,i}}{\sigma_i^{2}}+1 \right) +\sigma_i^{2}\sum_{j=-r_{2,i}}^{-1}\left(\frac{\lambda_{j,i}}{\sigma_i^{2}}+1 \right) \\&=p\sigma_i^{2}+\sum_{j=1}^{r_{1,i}}\lambda_{j,i}+\sum_{j=-r_{2,i}}^{-1}\lambda_{j,i},
	\end{align*}
	then we have 
	\begin{align}
	\sum_{j=1}^{p}l_{j,i}-p\sigma_i^{2}-\sum_{j=1}^{r_{1,i}}\lambda_{j,i}-\sum_{j=-r_{2,i}}^{-1}\lambda_{j,i}\stackrel{\mathcal{D}}{\rightarrow} \mathcal{N}\left(0, 2J_i\sigma^{4}\right).\label{all}
	\end{align} 
	For term $ \sum_{j=1}^{r_{1,i}}l_{j,i} $, $\sum_{j=-r_{2,i}}^{-1}l_{j,i}, $ since 
	\begin{align*}
	l_{j,i}\rightarrow\sigma_i^{2}\phi\left(  \frac{\lambda_{j,i}}{\sigma_i^2}+1 \right) =\lambda_{j,i}+\sigma_i^{2}+\sigma_i^{2}J_i\left(1+\frac{\sigma_i^{2}}{\lambda_{j,i}} \right)  ~~~\text{almost surely},
	\end{align*}	
	then we have
	\begin{align}
	\sum_{j=1}^{r_{1,i}}l_{j,i}+\sum_{j=-r_{2,i}}^{-1}l_{j,i}\rightarrow\sum_{j=1}^{r_{1,i}}\left(  \lambda_{j}+\sigma_i^{2}+\sigma_i^{2}J_i\left(1+\frac{\sigma_i^{2}}{\lambda_{j}} \right)\right) +\sum_{j=-r_{2,i}}^{-1}\left(  \lambda_{j}+\sigma_i^{2}+\sigma_i^{2}J_i\left(1+\frac{\sigma_i^{2}}{\lambda_{j}} \right)\right)   ~~~\text{almost surely}.\label{spike}
	\end{align}
	By results (\ref{all}) and (\ref{spike}) and Slutsky's lemma, we have
	\begin{align*}
	\frac{p-(r_{1,i}+r_{2,i})}{\sigma_i^{2}\sqrt{2J_i}}\left( \hat{\sigma}_i^{2}-\sigma_i^{2}\right) +\sqrt{\frac{J_i}{2}}\left((r_{1,i}+r_{2,i})+\sigma_i^{2}\left(\sum_{j=1}^{r_{1,i}}\frac{1}{\lambda_{j,i}}+\sum_{i=-r_{2,i}}^{-1}\frac{1}{\lambda_{j,i}} \right)  \right) \stackrel{\mathcal{D}}{\rightarrow} \mathcal{N}\left(0, 1\right),
	\end{align*}
	then the proof is complete.
\end{proof}



%
 
%

\bibliographystyle{IEEEtran}
\bibliography{trial}

\begin{thebibliography}{10}
\providecommand{\url}[1]{#1}
\csname url@samestyle\endcsname
\providecommand{\newblock}{\relax}
\providecommand{\bibinfo}[2]{#2}
\providecommand{\BIBentrySTDinterwordspacing}{\spaceskip=0pt\relax}
\providecommand{\BIBentryALTinterwordstretchfactor}{4}
\providecommand{\BIBentryALTinterwordspacing}{\spaceskip=\fontdimen2\font plus
\BIBentryALTinterwordstretchfactor\fontdimen3\font minus
  \fontdimen4\font\relax}
\providecommand{\BIBforeignlanguage}[2]{{%
\expandafter\ifx\csname l@#1\endcsname\relax
\typeout{** WARNING: IEEEtran.bst: No hyphenation pattern has been}%
\typeout{** loaded for the language `#1'. Using the pattern for}%
\typeout{** the default language instead.}%
\else
\language=\csname l@#1\endcsname
\fi
#2}}
\providecommand{\BIBdecl}{\relax}
\BIBdecl

\bibitem{0An}
T.~W. Anderson, \emph{An Introduction to Multivariate Statistical Analysis, 3rd
  Edition}.\hskip 1em plus 0.5em minus 0.4em\relax An introduction to
  multivariate statistical analysis.

\bibitem{li2015sparse}
Q.~Li and J.~Shao, ``Sparse quadratic discriminant analysis for high
  dimensional data,'' \emph{Statistica Sinica}, pp. 457--473, 2015.

\bibitem{fan2015innovated}
Y.~Fan, Y.~Kong, D.~Li, and Z.~Zheng, ``Innovated interaction screening for
  high-dimensional nonlinear classification,'' \emph{The Annals of Statistics},
  vol.~43, no.~3, pp. 1243--1272, 2015.

\bibitem{jiang2018direct}
B.~Jiang, X.~Wang, and C.~Leng, ``A direct approach for sparse quadratic
  discriminant analysis,'' \emph{Journal of Machine Learning Research},
  vol.~19, no.~31, pp. 1--37, 2018.

\bibitem{cai2021convex}
T.~T. Cai and L.~Zhang, ``A convex optimization approach to high-dimensional
  sparse quadratic discriminant analysis,'' \emph{The Annals of Statistics},
  vol.~49, no.~3, pp. 1537--1568, 2021.

\bibitem{gaynanova2019sparse}
I.~Gaynanova and T.~Wang, ``Sparse quadratic classification rules via linear
  dimension reduction,'' \emph{Journal of multivariate analysis}, vol. 169, pp.
  278--299, 2019.

\bibitem{friedman1989regularized}
J.~H. Friedman, ``Regularized discriminant analysis,'' \emph{Journal of the
  American statistical association}, vol.~84, no. 405, pp. 165--175, 1989.

\bibitem{ramey2016high}
J.~A. Ramey, C.~K. Stein, P.~D. Young, and D.~M. Young, ``High-dimensional
  regularized discriminant analysis,'' \emph{arXiv preprint arXiv:1602.01182},
  2016.

\bibitem{wu2019quadratic}
Y.~Wu, Y.~Qin, and M.~Zhu, ``Quadratic discriminant analysis for
  high-dimensional data,'' \emph{Statistica Sinica}, vol.~29, no.~2, pp.
  939--960, 2019.

\bibitem{wu2022quadratic}
R.~Wu and N.~Hao, ``Quadratic discriminant analysis by projection,''
  \emph{Journal of Multivariate Analysis}, vol. 190, p. 104987, 2022.

\bibitem{Bai04}
Z.~Bai and J.~W. Silverstein., ``Clt for linear spectral statistics of
  large-dimensional sample covariance matrices,'' \emph{The Annals of
  Probability}, vol.~32, no.~1A, pp. 553--605, 2004.

\bibitem{bai2012estimation}
Z.~Bai and X.~Ding, ``Estimation of spiked eigenvalues in spiked models,''
  \emph{Random Matrices: Theory and Applications}, vol.~1, no.~02, p. 1150011,
  2012.

\bibitem{Baik06}
J.~Baik and J.~W. Silverstein, ``Eigenvalues of large sample covariance
  matrices of spiked population models,'' \emph{Journal of Multivariate
  Analysis}, vol.~97, no.~6, pp. 1382--1408, 2006.

\bibitem{2011Random}
R.~Couillet and M.~Debbah, ``Random matrix methods for wireless communications:
  Random matrices,'' 2011.

\bibitem{sifaou2020high}
H.~Sifaou, A.~Kammoun, and M.-S. Alouini, ``High-dimensional linear
  discriminant analysis classifier for spiked covariance model?'' 2020.

\bibitem{li2022spectrally}
H.~Li, W.~Luo, Z.~Bai, H.~Zhou, and Z.~Pu, ``Spectrally-corrected and
  regularized linear discriminant analysis for spiked covariance model,''
  \emph{arXiv preprint arXiv:2210.03859}, 2022.

\bibitem{Sifaou2020HighDimensionalQD}
\BIBentryALTinterwordspacing
H.~Sifaou, A.~Kammoun, and M.-S. Alouini, ``High-dimensional quadratic
  discriminant analysis under spiked covariance model,'' \emph{IEEE Access},
  vol.~8, pp. 117\,313--117\,323, 2020. [Online]. Available:
  \url{https://api.semanticscholar.org/CorpusID:220055641}
\BIBentrySTDinterwordspacing

\bibitem{johnstone2001distribution}
I.~M. Johnstone, ``On the distribution of the largest eigenvalue in principal
  components analysis,'' \emph{Annals of statistics}, pp. 295--327, 2001.

\bibitem{hastie1995penalized}
T.~Hastie, A.~Buja, and R.~Tibshirani, ``Penalized discriminant analysis,''
  \emph{The Annals of Statistics}, vol.~23, no.~1, pp. 73--102, 1995.

\bibitem{laloux2000random}
L.~Laloux, P.~Cizeau, M.~Potters, and J.-P. Bouchaud, ``Random matrix theory
  and financial correlations,'' \emph{International Journal of Theoretical and
  Applied Finance}, vol.~3, no.~03, pp. 391--397, 2000.

\bibitem{malevergne210115collective}
Y.~Malevergne and D.~Sornette, ``Collective origin of the coexistance of
  apparent rmt noise and factors in large sample correlation matrices,''
  \emph{arXiv preprint cond-mat/0210115}.

\bibitem{plerou2002random}
V.~Plerou, P.~Gopikrishnan, B.~Rosenow, L.~A.~N. Amaral, T.~Guhr, and H.~E.
  Stanley, ``Random matrix approach to cross correlations in financial data,''
  \emph{Physical Review E}, vol.~65, no.~6, p. 066126, 2002.

\bibitem{kritchman2008determining}
S.~Kritchman and B.~Nadler, ``Determining the number of components in a factor
  model from limited noisy data,'' \emph{Chemometrics and Intelligent
  Laboratory Systems}, vol.~94, no.~1, pp. 19--32, 2008.

\bibitem{passemier2017estimation}
D.~Passemier, Z.~Li, and J.~Yao, ``On estimation of the noise variance in high
  dimensional probabilistic principal component analysis,'' \emph{Journal of
  the Royal Statistical Society: Series B (Statistical Methodology)}, vol.~79,
  no.~1, pp. 51--67, 2017.

\bibitem{telatar1999capacity}
E.~Telatar, ``Capacity of multi-antenna gaussian channels,'' \emph{European
  transactions on telecommunications}, vol.~10, no.~6, pp. 585--595, 1999.

\bibitem{sear2003instabilities}
R.~P. Sear and J.~A. Cuesta, ``Instabilities in complex mixtures with a large
  number of components,'' \emph{Physical review letters}, vol.~91, no.~24, p.
  245701, 2003.

\bibitem{davidson2009functional}
D.~J. Davidson, ``Functional mixed-effect models for electrophysiological
  responses,'' \emph{Neurophysiology}, vol.~41, no.~1, pp. 71--79, 2009.

\bibitem{fazli2011?1}
S.~Fazli, M.~Dan{\'o}czy, J.~Schelldorfer, and K.-R. M{\"u}ller, ``?1-penalized
  linear mixed-effects models for high dimensional data with application to
  bci,'' \emph{NeuroImage}, vol.~56, no.~4, pp. 2100--2108, 2011.

\bibitem{hoyle2003limiting}
D.~Hoyle and M.~Rattray, ``Limiting form of the sample covariance eigenspectrum
  in pca and kernel pca,'' \emph{Advances in Neural Information Processing
  Systems}, vol.~16, pp. 1181--1188, 2003.

\bibitem{friedman2001elements}
J.~Friedman, T.~Hastie, R.~Tibshirani \emph{et~al.}, \emph{The elements of
  statistical learning}.\hskip 1em plus 0.5em minus 0.4em\relax Springer series
  in statistics New York, 2001, vol.~1, no.~10.

\bibitem{zollanvari2015generalized}
A.~Zollanvari and E.~R. Dougherty, ``Generalized consistent error estimator of
  linear discriminant analysis,'' \emph{IEEE transactions on signal
  processing}, vol.~63, no.~11, pp. 2804--2814, 2015.

\bibitem{2020Estimating}
J.~Fan, J.~Guo, and S.~Zheng, ``Estimating number of factors by adjusted
  eigenvalues thresholding,'' \emph{Journal of the American Statistical
  Association}, pp. 1--33, 2020.

\bibitem{ke2021estimation}
Z.~T. Ke, Y.~Ma, and X.~Lin, ``Estimation of the number of spiked eigenvalues
  in a covariance matrix by bulk eigenvalue matching analysis,'' \emph{Journal
  of the American Statistical Association}, pp. 1--19, 2021.

\bibitem{Anderson1956}
T.~W. Anderson and H.~Rubin, ``Statistical inference in factor analysis,''
  \emph{Proceedings of the third Berkeley symposium on mathematical statistics
  and probability}, vol.~5, pp. 111--150, 1956.

\bibitem{2005Phase}
J.~Baik, G.~B. Arous, and S.~Peche, ``Phase transition of the largest
  eigenvalue for non-null complex sample covariance matrices,'' pp. 1643--1697,
  2005.

\bibitem{2022Spectrally}
H.~Li, W.~Luo, Z.~Bai, H.~Zhou, and Z.~Pu, ``Spectrally-corrected and
  regularized linear discriminant analysis for spiked covariance model,''
  \emph{arXiv e-prints}, 2022.

\bibitem{Bai10}
J.~C. Zhidong~Bai and J.~Yao., ``On estimation of the population spectral
  distribution from a high-dimensional sample covariance matrix,''
  \emph{Australian New Zealand Journal of Statistics}, vol.~52, no.~4, pp.
  423--437, 2010.

\end{thebibliography}
%
%
%
%
%
%
%
%
%

\end{document}